\documentclass[11pt]{article}
\usepackage[margin=1in]{geometry}
\usepackage{amsmath,amssymb,amsthm,mathtools}
\usepackage{booktabs}
\usepackage{graphicx}
\usepackage{adjustbox}
\usepackage{enumitem}
\usepackage{microtype}
\usepackage[round]{natbib}

\usepackage[colorlinks=true,linkcolor=black,citecolor=blue!50!black,urlcolor=blue!50!black]{hyperref}%
\usepackage{iftex}
\ifXeTeX
  \usepackage{fontspec}
\else
  \usepackage{times}
\fi
\usepackage{algorithm}
\usepackage{algpseudocode}
\usepackage{xcolor}

\theoremstyle{plain}
\newtheorem{theorem}{Theorem}
\newtheorem{proposition}{Proposition}
\newtheorem{corollary}{Corollary}
\newtheorem{lemma}{Lemma}
\theoremstyle{definition}
\newtheorem{assumption}{Assumption}
\newtheorem{remark}{Remark}

\DeclareMathOperator{\sign}{sign}
\DeclareMathOperator{\TV}{TV}
\DeclareMathOperator*{\argmin}{arg\,min}
\newcommand{\Real}{\mathbb{R}}
\newcommand{\Prob}{\mathbb{P}}

\title{Reliability-Weighted Direct Preference Optimization:\\
Penalized Weights for Robust Preference Learning}
\author{Bin Luo$^{1}$\quad Param Venkat Vivek Kesireddy$^{1}$\quad Anil Vallepu$^{1}$\\[4pt]
$^{1}$School of Data Science and Analytics, Kennesaw State University\\
Correspondence: \texttt{bluo@kennesaw.edu}}
\date{October 2026}

\begin{document}

\maketitle

\begin{abstract}
Direct Preference Optimization (DPO) treats every reported comparison as correct, yet preference data
contain reversed, tied or otherwise noisy comparisons. RW-DPO gives each comparison an unknown reliability weight
$w_i\in[0,1]$ that enters the likelihood linearly --- DPO is $w\equiv1$, a reversed comparison $w_i=0$ --- and
estimates the weights jointly with the policy under a minimax-concave penalty on $-\log w_i$. The penalty keeps
every comparison the model agrees with at exactly $w_i=1$ and acts only on those it contradicts, through
closed-form weight updates at negligible cost over DPO. We prove that recovering all flipped comparisons from a single
annotation per comparison requires margins of order $\log n$. Furthermore, in a linear-index model, starting from an
initial fit pointing in the correct direction, our estimator jointly consistently recovers the preference direction
and weights, as well as identifies the flipped comparisons above this margin scale, even when a fixed fraction of
comparisons is flipped. In simulations RW-DPO has the lowest average
estimation error among DPO, noise-rate corrections and recent robust methods, with far more precise flags than
recent label-correction methods. In DPO fine-tuning of a language model with $10$--$30\%$
flipped labels, its win rate does not fall as the flip rate grows and exceeds that of DPO trained on clean
labels, matching the strongest recent method with corrections concentrated on the flipped labels.
\end{abstract}

\section{Introduction}

Direct Preference Optimization \citep{rafailov2023} aligns a language model from pairwise comparisons and
treats every reported preference as correct. Real comparisons are not: some are \emph{reversed} (the annotator
prefers the inferior response), some \emph{tied} (the interface forces an arbitrary choice), and many
\emph{noisy}: the label matches the true preference only with some probability below one. DPO fits them all, and its quality falls steadily with the share of flipped
labels (Section~\ref{sec:exp-rq2}).

Existing remedies fall short in two ways (Section~\ref{sec:related}). Methods built on a global flip rate
\citep{chowdhury2024rdpo,mitchell2023cdpo} treat every comparison alike, so they cannot tell which comparisons are
wrong. Methods that act on each comparison each lack part of what is needed: ROPO \citep{liang2024ropo} can discard a
comparison but not reverse it, PLC-DPO \citep{cho2026plc} blends every comparison with its reversal by a fixed rule
and so never fully trusts one, and R3M's $\ell_1$ penalty \citep{bukharin2024r3m} biases the estimates. None comes
with a theoretical guarantee that the flipped comparisons are recovered when a fixed fraction of labels is flipped,
the regime of real preference data.

RW-DPO gives each comparison a reliability weight, which lets the fitted model correct the labels it
contradicts. The weight enters the likelihood linearly, so it moves continuously from trusted ($w_i=1$) through
uninformative ($w_i=\tfrac12$) to reversed ($w_i=0$). The weights are estimated jointly with the policy under a
minimax-concave penalty on $-\log w_i$, the scale on which a trusted comparison sits at the origin, which makes
them sparse on that scale: most comparisons, the correct ones, keep exactly $w_i=1$ and train as in DPO. A
comparison the model contradicts is down-weighted, the more the stronger the contradiction; and because the penalty
flattens beyond a saturation point, a strongly contradicted comparison can reach exactly $w_i=0$ at a finite price
and is then reversed.

Our contributions are:
\begin{enumerate}[leftmargin=*,nosep]
\item \textbf{A reliability model} (Section~\ref{sec:model}): an unknown $w_i\in[0,1]$ enters the likelihood
linearly and nests DPO, reversal, ties ($w_i=\tfrac12$) and partial reliability.
\item \textbf{A penalty on $\boldsymbol{-\log w}$} (Section~\ref{sec:method}) that gives exact $w_i=1$, graded
attenuation and exact $w_i=0$, with closed-form weight updates and weighted DPO as policy step (Section~\ref{sec:alg}).
\item \textbf{Guarantees} (Section~\ref{sec:theory}): the information limit for detecting a flip from one
annotation and, in a linear-index model, joint consistency of the direction and the weights with consistent
flip recovery.
\item \textbf{Evidence} (Section~\ref{sec:exp}) in simulations and in DPO fine-tuning of a language model.
\end{enumerate}

\section{Related work}
\label{sec:related}

\begin{table}[t]
\centering
\footnotesize
\setlength{\tabcolsep}{2.5pt}
\caption{RW-DPO against robust DPO variants: how finely reliability is resolved, whether a flip can be reversed,
and what is guaranteed.}
\label{tab:compare}
\begin{adjustbox}{max width=\columnwidth}\begin{tabular}{@{}lllll@{}}
\toprule
Method & Reliability & Reverses & Detection & Consistency \\
       & resolved    & a flip   & limit     &             \\
\midrule
DPO & none & --- & --- & --- \\
rDPO \citep{chowdhury2024rdpo} & global rate & on average & --- & known rate \\
Dr.\ DPO \citep{wu2024drdpo} & global (DRO) & --- & --- & DRO bound \\
ROPO \citep{liang2024ropo} & per-pair weight & no & --- & convergence\textsuperscript{a} \\
FA-DPO \citep{xu2025fadpo} & covariate model & \checkmark & --- & if model holds \\
H\"older-DPO \citep{fujisawa2025holder} & robust loss & --- & --- & robustness bound \\
R3M \citep{bukharin2024r3m} & mean shift ($\ell_1$) & absorbs & --- & sublinear outliers \\
PLC-DPO \citep{cho2026plc} & per-pair route & \checkmark & --- & --- \\
\textbf{RW-DPO} & per-pair weight & \checkmark\ (exact) & \checkmark & \checkmark\textsuperscript{b} \\
\bottomrule
\end{tabular}\end{adjustbox}

\vspace{2pt}
{\scriptsize \textsuperscript{a}Convergence of the algorithm, not statistical consistency. \textsuperscript{b}Linear-index
model, binary true weights, for flips above a detection threshold (Section~\ref{sec:theory}).\par}
\end{table}

\paragraph{Global corrections.} Conservative DPO \citep{mitchell2023cdpo} smooths labels, rDPO
\citep{chowdhury2024rdpo} debiases the loss under a known flip rate, and Dr.\ DPO \citep{wu2024drdpo} hedges
against worst-case noise. All three treat every comparison alike.

\paragraph{Per-comparison corrections.} ROPO \citep{liang2024ropo} learns non-negative weights under a retained
budget, discarding but never reversing; FA-DPO \citep{xu2025fadpo} models the flip probability by a logistic
regression on covariates; H\"older-DPO \citep{fujisawa2025holder} uses a redescending loss that also locates
mislabels. PLC-DPO \citep{cho2026plc} routes each pair to clean, flipped or tied states by a fixed posterior formula
of its standardized margin and mixes the corresponding losses, so every comparison receives some weight on its
reversed label. Closest in structure, R3M \citep{bukharin2024r3m} adds an $\ell_1$-penalized mean shift per
comparison and extends to DPO. RW-DPO instead estimates the weights by penalized likelihood: most stay at exactly
$w_i=1$, MCP avoids the $\ell_1$ bias \citep{sheowen2011}, and reversal is exact.

\paragraph{What is guaranteed.} The theory behind these methods concerns the policy or the optimization, not which
comparisons are flipped (Table~\ref{tab:compare}). rDPO bounds the policy's error when the flip rate is known,
Dr.\ DPO bounds a worst-case risk and H\"older-DPO the effect of outliers; ROPO's guarantee is convergence of its
algorithm, FA-DPO is consistent only when its model of the flip probability is correct, and PLC-DPO gives none.
R3M's theory, the closest to ours, addresses the corrupted comparisons themselves but requires their number to grow
sublinearly in $n$, so the share of flips must vanish. In preference data the share of wrong labels does not
shrink as more data are collected. RW-DPO's guarantees are stated for that regime: with a fixed fraction of flips,
the estimated direction and weights are consistent, and the flips beyond a detection threshold are flagged and,
with a shrinking cut, reversed; the threshold is set by what one annotation can reveal
(Theorems~\ref{thm:info-recovery}--\ref{thm:recovery}).

\paragraph{Robust regression and case weights.} \citet{sheowen2011} add per-observation mean-shift parameters
under a nonconvex penalty and show that an $L_1$ penalty loses the consistency the nonconvex one attains --- our
reason for MCP; \citet{gaofang2016} recast this as log-scale case weights, the device we adopt. Case-specific and
misclassification-robust logistic fits \citep{tibsmanning2014,lee2015case,copas1988,carrollpederson1993} and
folded-concave theory \citep{fanli2001scad,zou2008lla,zhang2010mcp,fanxuezou2014} identify the per-case parameter
from a continuous residual, not from one binary annotation.

\section{Background: Direct Preference Optimization}
\label{sec:background}

DPO replaces the reward of KL-regularized RLHF by $\beta\log(\pi_\theta/\pi_{\mathrm{ref}})$, whose normalizing
constant cancels within a pair. We fix a label-independent order on each pair (e.g.\ by hashing the responses),
write the responses $y_i^A,y_i^B$, and set $Y_i=+1$ if the annotator reported $y_i^A\succ y_i^B$ and $-1$ otherwise.

\paragraph{Margin.} With policy $\pi_\theta$, frozen reference
$\pi_{\mathrm{ref}}$ and temperature $\beta$, define the DPO margin
\begin{equation}
\label{eq:margin}
d_i(\theta)=\beta\left[\log\frac{\pi_\theta(y_i^A\mid x_i)}{\pi_{\mathrm{ref}}(y_i^A\mid x_i)}
-\log\frac{\pi_\theta(y_i^B\mid x_i)}{\pi_{\mathrm{ref}}(y_i^B\mid x_i)}\right].
\end{equation}
Under Bradley--Terry, $\Prob(Y_i=+1)=\sigma(d_i)$ and DPO minimizes the negative log-likelihood
$-\sum_i\log\sigma(Y_id_i)$. By $\sigma(t)=\tfrac{1+\tanh(t/2)}{2}$, we may write
\begin{equation}
\label{eq:tanh}
\Prob(Y_i=\pm1)=\tfrac{1\pm g(d_i)}{2},\qquad g(d)=\tanh(d/2)\in(-1,1),
\end{equation}
since $g$ is exactly what a reversed label negates. We write
\begin{equation}
\label{eq:m}
m_i(\theta)=Y_i\,g\big(d_i(\theta)\big)\in(-1,1)
\end{equation}
for the \emph{agreement} between model and annotator ($m_i$ near $-1$: the model confidently
contradicts the label).

\section{A contamination model with per-comparison reliability}
\label{sec:model}

Let $p_i^+$ denote the probability that the \emph{true} preference points in the
direction the annotator reported, and $p_i^-=1-p_i^+$ the probability that it
points the other way, so that $p_i^\pm=\tfrac{1\pm Y_ig(d_i)}{2}$ by
\eqref{eq:tanh}. Attach to comparison $i$ a \textbf{reliability weight}
$w_i\in[0,1]$ and model the reported label as
\begin{equation}
\label{eq:mix}
\Prob(Y_i \text{ as observed}) \;=\; w_i\,p_i^+ + (1-w_i)\,p_i^- .
\end{equation}
The interpretation is direct: with probability $w_i$ the annotation transmits the
true preference, and with probability $1-w_i$ it transmits its opposite.

\begin{lemma}[Single-index reduction]
\label{lem:reduce}
With $m_i$ as in \eqref{eq:m}, the model \eqref{eq:mix} is equivalent to
\begin{equation}
\label{eq:r}
r_i(\theta,w_i)\;:=\;\Prob(Y_i\text{ as observed})\;=\;\tfrac12+\big(w_i-\tfrac12\big)\,m_i(\theta) .
\end{equation}
\end{lemma}

So $w$ enters the likelihood \emph{linearly}, multiplying the agreement $m_i$.
This is the structural feature the method exploits, and it makes the four regimes
of interest a single continuum:

$w_i=1$ is the DPO term; $w_i\in(\tfrac12,1)$ down-weights the pair; $w_i=\tfrac12$ makes it
uninformative ($r_i=\tfrac12$ whatever $\theta$); $w_i<\tfrac12$ corrects its direction, fully at $w_i=0$.
A genuine tie corresponds to $w_i=\tfrac12$, but ties need not be the only source of
intermediate reliability.

With one weight per comparison the model is saturated; what identifies $\theta$ is that most
comparisons are reliable ($w_i^*=1$), and the penalty concentrates on the minority that does not fit.
In the guarantees of Section~\ref{sec:theory} this requirement enters through the initial fit
(Assumption~\ref{ass:init}).

Lemmas~\ref{lem:reduce}--\ref{lem:wstep}, the weight-step results of Appendix~\ref{app:wstep} and
the first part of Proposition~\ref{prop:bayes} hold for any link in
$\mathcal G=\{g:\Real\to(-1,1)\text{ odd, increasing, }g(0)=0\}$; Theorem~\ref{thm:info-recovery} and
the analysis of Section~\ref{sec:feasible} use the Bradley--Terry link $g(d)=\tanh(d/2)$ of
\eqref{eq:tanh}, which is also our default.

\section{Method: penalizing the log-weight}
\label{sec:method}

RW-DPO estimates the weights jointly with the policy by minimizing the penalized negative log-likelihood of
\eqref{eq:r},
\begin{equation}
\label{eq:obj}
\begin{split}
\Psi(\theta,w)={}&-\sum_{i=1}^{n}\log\Big[\tfrac12+\big(w_i-\tfrac12\big)m_i(\theta)\Big]\\
&+\sum_{i=1}^{n}P(w_i),\qquad w\in[0,1]^n,
\end{split}
\end{equation}
where $P$ penalizes departures from full reliability. Fixing $w\equiv1$ gives $\Psi=-\sum_i\log\sigma(Y_id_i)$, the
DPO objective. With a common weight $w_i\equiv1-\varepsilon$ the likelihood part of \eqref{eq:obj} is the exact
likelihood under a known flip rate $\varepsilon$ (with one weight per annotator, Crowd-BT;
\citealp{chen2013crowdbt}): cDPO's loss \citep{mitchell2023cdpo} is its Jensen upper bound and rDPO's
\citep{chowdhury2024rdpo} an unbiased correction. RW-DPO lets the weight vary by comparison. Nothing requires a
particular policy class, since $m_i$ is computed from the neural margin
\eqref{eq:margin}. The estimator is the alternating update of Section~\ref{sec:alg} built from $\Psi$.

Writing $u_i=-\log w_i\in[0,\infty]$ puts full reliability at the origin, so any penalty
$\varrho(u_i)$ with $\varrho(0)=0$ that grows with $u_i$ is free at $w_i=1$, and a sparsity-inducing $\varrho$
keeps most $u_i$ at exactly zero, that is most $w_i$ at exactly one --- the device of \citet{gaofang2016} for
robust regression, transplanted to a weight that multiplies an agreement rather than a residual.

For $\varrho$ we take the minimax-concave penalty (MCP)
\citep{zhang2010mcp},
\begin{equation}
\label{eq:mcp}
\varrho(u;\lambda,a)=
\begin{cases}
\lambda u-\dfrac{u^{2}}{2a}, & 0\le u\le a\lambda,\\[2mm]
\dfrac{a\lambda^{2}}{2}, & u> a\lambda,
\end{cases}
\end{equation}
with $\lambda,a>0$, and define the penalty on the weight scale by composition,
\begin{equation}
\label{eq:pen}
P(w)=\varrho(-\log w;\lambda,a),\qquad P(0)=\frac{a\lambda^{2}}{2}.
\end{equation}
Like the lasso, the MCP has slope $\lambda$ at $u=0$, which makes $w_i=1$ an exact solution: a comparison the
model agrees with, or only mildly contradicts, keeps its full weight. Unlike the lasso, whose constant slope keeps
shrinking large departures and biases them \citep{sheowen2011}, its slope falls to zero at $u=a\lambda$: the
curvature in between down-weights a contradicted comparison in proportion to the evidence, and beyond the
\emph{saturation weight} $w_{\mathrm{sat}}=e^{-a\lambda}$ further loss of reliability costs nothing, so a full
reversal ($w_i=0$) costs only the finite amount $a\lambda^2/2$. Section~\ref{sec:alg} shows how the algorithm
reaches each of the three states.

\section{Algorithm}
\label{sec:alg}

We alternate a $w$-step, separable across comparisons and in closed form after one standard
approximation, and a $\theta$-step, an ordinary weighted DPO update.

\subsection{The \texorpdfstring{$w$}{w}-step}

With $\theta$ fixed, \eqref{eq:obj} separates into scalar problems
\begin{equation}
\label{eq:wsub}
\min_{w\in[0,1]}\ -\log\Big[\tfrac12+\big(w-\tfrac12\big)m_i\Big]+P(w).
\end{equation}
Their stationarity condition is transcendental, so we use the local linear approximation standard for
folded-concave penalties \citep{zou2008lla}, replacing $\varrho(-\log w)$ at the current $w_i^{(t)}$ by its
tangent majorant
\begin{equation}
\label{eq:lla}
\varrho(-\log w)\approx\varpi_i^{(t)}(-\log w)+\text{const},
\end{equation}
with $\varpi_i^{(t)}=(\lambda+a^{-1}\log w_i^{(t)})_+$: a weighted log-penalty problem with a closed-form solution.

\begin{lemma}[Closed-form weight update]
\label{lem:wstep}
Fix $m\in(-1,1)$ and $\varpi\ge0$, and let
$h(w)=-\log[\tfrac12+(w-\tfrac12)m]+\varpi(-\log w)$, with $\hat w=1$ if $h$ is constant. Then
\[
\argmin_{[0,1]}h=
\begin{cases}
1, & m\ge0,\\[1mm]
\min\Big\{1,\dfrac{\varpi(1-m)}{2|m|(1+\varpi)}\Big\}, & m<0<\varpi,\\[3mm]
0, & m<0=\varpi .
\end{cases}
\]
In particular $\hat w=1$ exactly whenever $m\ge-\dfrac{\varpi}{2+\varpi}$, and
$\hat w=0$ exactly whenever $\varpi=0$ and $m<0$.
\end{lemma}

So the update keeps $\hat w_i=1$ whenever the model agrees with the annotation or disagrees only mildly,
that is, when $m_i\ge-\varpi_i/(2+\varpi_i)$. On the first pass all weights start at $w_i=1$, giving $\varpi_i=\lambda$
and the threshold $-\lambda/(2+\lambda)$, below which the update returns an interior weight that decreases as the
disagreement grows. Because weights are carried across passes and $\varpi_i$ falls
with $w_i$, a down-weighted comparison is penalized less and can fall further; once
$w_i\le w_{\mathrm{sat}}$ the next update sets $\hat w_i=0$, an exact reversal. Whether this happens or the
weight settles at an interior value is decided by $|m_i|$ and $(\lambda,a)$
(Lemma~\ref{lem:lla-limit}, Appendix~\ref{app:wstep}).

\subsection{A scale-free agreement}
\label{sec:zscale}

The margin scale of a trained policy grows as it sharpens, and down-weighting makes each refit's margins
larger still, so a fixed $\lambda$ on raw margins flags an ever larger share of comparisons
The problem is common to margin-based corrections: PLC-DPO \citep{cho2026plc}
likewise standardizes its margins, by running estimates of their mean and standard deviation. RW-DPO therefore
computes the weights on a \emph{standardized} margin, with a robust scale and without centring. With the label-oriented margin $\xi_i=Y_id_i(\theta)$ (so that $m_i=g(\xi_i)$ in \eqref{eq:m}), let
\[
\hat s=1.4826\,\mathrm{med}_i\big|\xi_i-\mathrm{med}_j\,\xi_j\big|,\quad
z_i=\frac{\xi_i}{\hat s},\quad \tilde m_i=g(z_i),
\]
the robust standard deviation (MAD) of the label-oriented margins, and run the $w$-step of
Lemma~\ref{lem:wstep} on $\tilde m_i$; the margins are not centred, so $z_i<0$ still means disagreement, and
$\beta$ cancels. The $\theta$-step keeps the raw margins. With $g(z)=\tanh(z/2)$, on the first pass a comparison is
down-weighted ($\hat w_i<1$) exactly when $z_i<-c$, with
\[
\tanh\tfrac c2=\frac{\lambda}{2+\lambda}\qquad\Longleftrightarrow\qquad \lambda=e^{c}-1 ,
\]
so we parametrize the penalty by the \emph{cut} $c$, in robust standard deviations, and the saturation
weight $w_{\mathrm{sat}}$ (Appendix~\ref{sec:tuning}). The standardization is also what lets the theory
require only the \emph{direction} of the initial fit (Section~\ref{sec:feasible}).

\subsection{The \texorpdfstring{$\theta$}{theta}-step and the full procedure}

With $w$ fixed, the $\theta$-step is a DPO gradient step in which comparison $i$ carries the coefficient
$(w_i-\tfrac12)g'(d_i)Y_i/r_i$ (Appendix~\ref{app:details}): the DPO weight at $w_i=1$, zero at
$w_i=\tfrac12$, and of opposite sign below $\tfrac12$, so a reversed comparison pushes the policy in the
corrected direction. The objective is invariant under negating every label and complementing every
weight; initializing at the DPO solution selects the label-trusting basin (Appendix~\ref{app:details}).

\begin{algorithm}[t]
\caption{RW-DPO}
\label{alg:rwdpo}
\begin{algorithmic}[1]
\Require preference data $\{(x_i,y_i^A,y_i^B,Y_i)\}_{i=1}^n$; reference
$\pi_{\mathrm{ref}}$; $\beta$; cut $c$ and saturation weight $w_{\mathrm{sat}}$ (so $\lambda=e^c-1$, $a=-\log w_{\mathrm{sat}}/\lambda$); link $g\in\mathcal G$
\State cache the reference log-probabilities
\State $\theta^{(0)}\gets$ DPO fit \Comment{selects the basin}
\State $w_i^{(0)}\gets1$ for all $i$
\For{$t=0,1,2,\dots$}
  \State $\xi_i\gets Y_id_i(\theta^{(t)})$, $\hat s\gets1.4826\,\mathrm{MAD}(\xi)$
  \State $m_i\gets g(\xi_i/\hat s)$ \Comment{$w$-step (Sec.~\ref{sec:zscale})}
  \State $\varpi_i\gets\big(\lambda+a^{-1}\log w_i^{(t)}\big)_+$
  \State $w_i^{(t+1)}\gets1$ if $m_i\ge0$, else
  \Statex \hspace{\algorithmicindent}$\min\{1,\varpi_i(1-m_i)/[2|m_i|(1+\varpi_i)]\}$
  \State $\theta^{(t+1)}\gets$ weighted DPO steps \Comment{raw margins}
\EndFor
\State \Return $\hat\theta$, $\hat w$, and the partition
$\{w_i=1\}$, $\{0<w_i<1\}$, $\{w_i=0\}$
\end{algorithmic}
\end{algorithm}

Algorithm~\ref{alg:rwdpo} costs one scalar formula per comparison per pass over DPO. Like iteratively
reweighted least squares with a MAD scale, it is defined by its update, which Section~\ref{sec:feasible} analyses
directly.

\section{Theoretical analysis}
\label{sec:theory}
Two questions decide whether per-comparison correction can work: how much one annotation can reveal about a
flip, and whether an estimator that learns the weights and the policy together keeps both on track.
Section~\ref{sec:info} answers the first with a sharp limit. Section~\ref{sec:feasible} answers the second for
the full-data alternating estimator, with the $\theta$-step solved to optimality, in a linear-index model with binary true weights: the direction stays consistent
(Theorem~\ref{thm:joint}) and the detectable flips are recovered (Theorem~\ref{thm:recovery}), with rates in
Corollary~\ref{cor:rates}. Supporting results are in Appendices~\ref{app:wstep} and~\ref{app:remarks}, proofs in
Appendix~\ref{app:proofs}.

\subsection{Information limit for reliability detection}
\label{sec:info}

At the true parameter write $d_i^*=d_i(\theta^*)$ and $g_i^*=g(d_i^*)$, so
$\Prob(Y_i=+1)=\tfrac12+(w_i^*-\tfrac12)g_i^*$; the endpoint states $w_i^*=1$ and $w_i^*=0$
give $Y_i=+1$ with probability $\tfrac{1\pm g_i^*}{2}$ (hypotheses $H_{1i}$, $H_{0i}$).

\begin{proposition}[Bayes error for endpoint flip detection]
\label{prop:bayes}
Suppose $\theta^*$ is known and consider testing $H_{1i}:w_i^*=1$ versus
$H_{0i}:w_i^*=0$ from one binary annotation, with equal prior probabilities. The
minimum probability of classification error is
\begin{equation}
\label{eq:bayes-error}
e_i^*=\frac{1-|g_i^*|}{2}.
\end{equation}
For the Bradley--Terry link $g(d)=\tanh(d/2)$, this becomes
\begin{equation}
\label{eq:bayes-error-bt}
e_i^*=\frac{1}{1+\exp(|d_i^*|)}.
\end{equation}
\end{proposition}

So no estimator can reliably classify a comparison with bounded margin; since
$e_i^*\asymp e^{-|d_i^*|}$, exact recovery needs margins of order $\log n$.

\begin{theorem}[Information-theoretic exact recovery]
\label{thm:info-recovery}
Suppose the comparisons are conditionally independent given their covariates,
true margins and endpoint labels $w^*$. If
\begin{equation}
\label{eq:info-condition}
\min_{1\le i\le n}|d_i^*|\ge(1+\epsilon)\log n
\end{equation}
for some $\epsilon>0$, then the coordinatewise Bayes-optimal tests recover all
endpoint labels with probability tending to one:
$\Prob(\widehat{\mathcal F}_{\mathrm{Bayes}}=\mathcal F^*)\to1$, where
$\mathcal F^*=\{i:w_i^*=0\}$. Conversely, let the endpoint labels be drawn
independently with $\Prob(w_i^*=0)=\tfrac12$ (the equal-prior setting of
Proposition~\ref{prop:bayes}). If there is a constant $M<\infty$ and a set
$\mathcal I_n$ with $|\mathcal I_n|\ge c_In$ for some $c_I>0$ such that
$|d_i^*|\le M$ for $i\in\mathcal I_n$, then every estimator $\widehat{\mathcal F}$
built from one annotation per comparison satisfies
$\Prob(\widehat{\mathcal F}=\mathcal F^*)\to0$.
\end{theorem}

\subsection{Consistency of the feasible estimator in the linear-index model}
\label{sec:feasible}

We now analyse the full-data alternating estimator --- weights from the \emph{fitted} model's
standardized margins (Section~\ref{sec:zscale}), alternated with the $\theta$-step, here solved to optimality;
the simulations of Section~\ref{sec:exp-linear} solve it by L-BFGS to convergence, with a small ridge in place of
the trust region (Appendix~\ref{app:sim-details}) --- in the
linear-index model $d_i(\theta)=\phi_i^\top\theta$, $\phi_i\in\Real^p$, $p$ fixed.

\paragraph{Setting.} Write the true parameter as $\theta^*_n=\kappa_n\vartheta$ with
$\|\vartheta\|=1$, so $\kappa_n$ is the signal strength, and put
$v_i=\phi_i^\top\vartheta$, so the true margin is $d_i^*=\kappa_n v_i$. A clean
annotation $\widetilde Y_i\in\{\pm1\}$ is drawn from the Bradley--Terry model,
$\Prob(\widetilde Y_i=+1\mid\phi_i)=\sigma(d_i^*)$, independently over $i$ given
the features; the reported label is $Y_i=\widetilde Y_i$ if $w_i^*=1$ and
$Y_i=-\widetilde Y_i$ if $w_i^*=0$. For a direction we write
$\angle(\theta,\vartheta)=\arccos(\theta^\top\vartheta/\|\theta\|)$, and we use $\sign(0)=0$.

\begin{assumption}[Linear index, growing signal, design]
\label{ass:linear}
The design satisfies the following.
\begin{enumerate}[label=(\roman*),leftmargin=*,nosep]
\item $\phi_1,\dots,\phi_n$ are i.i.d., $\|\phi_i\|\le B$, and $p$ is fixed.
\item $\kappa_n/\log n\to\infty$ and $\kappa_n\le n$.
\item $\Prob(|v|\le b)\le \bar f\,b$ for $0\le b\le b_0$ (little mass near the boundary).
\item \emph{(direction is identified by the sign pattern)}
$e(u):=\Prob\{\sign(\phi^\top u)\ne\sign(\phi^\top\vartheta)\}\ge c_*\,\angle(u,\vartheta)$
for every $u\ne0$, with $c_*>0$.
\item $\phi$ has a density with respect to Lebesgue measure on $\Real^p$.
\end{enumerate}
\end{assumption}

\begin{assumption}[Endpoint contamination]
\label{ass:endpoint}
$w_i^*\in\{0,1\}$; the flip set $\mathcal F^*=\{i:w_i^*=0\}$ is fixed given the
features and is otherwise arbitrary (it may depend on the features). The
estimator itself is unrestricted, $w_i\in[0,1]$.
\end{assumption}

Conditions (iii)--(v) hold, e.g., for a design uniform on a ball, where $e(u)=\angle(u,\vartheta)/\pi$.

\paragraph{The estimator.} For $\theta\ne0$ let $\xi_i(\theta)=Y_i\phi_i^\top\theta$,
let $\hat s(\theta)=1.4826\,\mathrm{med}_i\,|\xi_i(\theta)-\mathrm{med}_j\,\xi_j(\theta)|$
be the robust scale of Section~\ref{sec:zscale} (any median convention with at least half the values on each side, e.g.\ either middle order
statistic or their average; the bounds use only that a median lies between the smallest and the largest value, and the
positivity of $\hat s$ uses the half-on-each-side property), and $z_i(\theta)=\xi_i(\theta)/\hat s(\theta)$
(not centred). The penalty
constants $(\lambda_n,a_n)$ are indexed by a \emph{cut} $c_n>0$ and a fixed saturation
weight $w_{\mathrm{sat}}\in(0,1)$,
$\lambda_n=e^{c_n}-1$ (so $\tfrac{\lambda_n}{2+\lambda_n}=\tanh\tfrac{c_n}{2}$) and
$a_n=\log(1/w_{\mathrm{sat}})/\lambda_n$,
so on the first pass a comparison is down-weighted ($\hat w_i<1$) exactly when $z_i<-c_n$. We analyse two weight steps at the agreements $\tilde m_i=\tanh\{z_i(\theta)/2\}$ from the current
weights $w$: one LLA update, $W_1(\theta,w)$, and its limit at fixed $\theta$, $W_\infty(\theta,w)$. The $\theta$-step returns a global minimizer
of the likelihood part of \eqref{eq:obj} over a trust region,
$\hat\theta(w)\in\argmin_{\|\theta\|\le R_n}L_n(\theta;w)$, where
\[
L_n(\theta;w)=-\frac1n\sum_{i=1}^n\log\Big[\tfrac12+\big(w_i-\tfrac12\big)\tanh\tfrac{Y_i\phi_i^\top\theta}{2}\Big],
\]
on the \emph{raw} margins, with $R_n\ge\kappa_n$ (e.g.\ $R_n=n$); the ball plays the role of the KL
anchor. Lemma~\ref{lem:thetastep}, Theorem~\ref{thm:joint} and Theorem~\ref{thm:recovery}(a)--(b) use only one
property of this step: the refit's loss is at most that of the competitor $\theta_{\mathrm{ref}}=\{\log(1/b)/b\}\vartheta$
of Lemma~\ref{lem:thetastep}, $L_n(\hat\theta;w)\le L_n(\theta_{\mathrm{ref}};w)$, which any $\theta$-step reaching
that loss satisfies. Global minimality is used only in Theorem~\ref{thm:recovery}(c) and
Corollary~\ref{cor:oracle}, whose conclusions concern the minimizer itself. From an initial fit
$\theta^{(0)}\ne0$ and $w^{(0)}=\mathbf 1$, RW-DPO iterates
$w^{(k)}=W(\theta^{(k-1)},w^{(k-1)})$, $\theta^{(k)}=\hat\theta(w^{(k)})$, $k=1,2,\dots$, with
$W\in\{W_1,W_\infty\}$ ($W_\infty$ is well defined from any start, since the update is monotone).

\begin{assumption}[Tuning and initial fit]
\label{ass:init}
$c_n\to0$; and $\angle(\theta^{(0)},\vartheta)\to0$ in probability.
\end{assumption}

The initial fit need only point in the right direction; plain DPO does so when flips are
independent of the features and the design is symmetric, whereas a feature-dependent minority of
flips can break it (Remark~\ref{rem:init}, Appendix~\ref{app:remarks}).

\paragraph{Detectable comparisons.} Put $\alpha_n=\sqrt{\lambda_n}$, the largest weight a detected flip can keep
after one pass, and $c_n^+=2\,\mathrm{artanh}\,\alpha_n$, the disagreement that guarantees it; both tend to $0$ with $c_n$. For a direction error
$\delta\ge0$ let
\[
b_n(\delta)=\max\Big\{\frac{(1+\epsilon)\log n}{\kappa_n},\,B\delta+3Bc_n^+\Big\},
\]
$\mathcal D_n(\delta)=\{i:|v_i|>b_n(\delta)\}$
for a fixed $\epsilon>0$, and $\gamma_n(\delta)=|\mathcal D_n(\delta)^c|/n$: the first term is the
information-theoretic scale of Theorem~\ref{thm:info-recovery}, the second the cost of the fit's
direction error and the cut. Let the event that every clear-cut clean label is correct be
$G_n=\{\widetilde Y_i=\sign(v_i)\text{ for every }i\text{ with }\kappa_n|v_i|\ge(1+\epsilon)\log n\}$;
by the union bound $\Prob(G_n^c)\le n^{-\epsilon}$.

The proofs rest on two lemmas (Appendix~\ref{app:remarks}): at a fit within angle $\delta$ of $\vartheta$ the
weight step is correct on $\mathcal D_n(\delta)$, and with such weights the $\theta$-step misorders few comparisons.

\begin{theorem}[Estimation consistency, non-asymptotic]
\label{thm:joint}
Let Assumptions~\ref{ass:linear} and~\ref{ass:endpoint} hold and write $\delta_k=\angle(\theta^{(k)},\vartheta)$. Fix $K\ge1$, $\zeta\in(0,1)$ and
$\bar\delta_0>0$, let $\eta_n=C_pn^{-1/2}+\{\log(2/\zeta)/(2n)\}^{1/2}$ with $C_p$ depending only on $p$, and set
$\bar\delta_k=F_n(\bar\delta_{k-1})$, where
\[
F_n(\delta)=\frac{(1+\bar f)\,b_n(\delta)\log\frac{4}{b_n(\delta)}+2\alpha_n+\eta_n\log(8n)}{c_*\log2}.
\]
Suppose $n\ge2p+2$ (ensures $\hat s>0$, and $n\ge3$), $\lambda_n\le w_{\mathrm{sat}}^2$, $\alpha_n<\frac12$, $b_n(\bar\delta_k)\le\min\{e^{-1},b_0\}$
for $0\le k<K$ and $\bar\delta_k<\pi$ for $k\le K$. Then for either weight step there is an event $E_n$ with
$\Prob(E_n)\ge1-n^{-\epsilon}-2\zeta-\Prob(\delta_0>\bar\delta_0)$ on which, for $1\le k\le K$,
\[
\delta_k\le\bar\delta_k,\qquad
\tfrac1n\textstyle\sum_i|w_i^{(k)}-w_i^*|\le\bar f\,b_n(\bar\delta_{k-1})+\eta_n+\alpha_n .
\]
Under Assumption~\ref{ass:init}, for suitable $\bar\delta_{0,n},\zeta_n\to0$, $\Prob(E_n)\to1$ and both bounds
tend to zero.
\end{theorem} Passes keep a consistent direction and correct the weights. The bounds $\bar\delta_k$
themselves grow with $k$: Assumption~\ref{ass:linear}(iii)--(iv) force $\bar f\ge c_*/B$, so
$F_n(\delta)\ge\delta\log(4e)/\log2\approx3.4\,\delta$ whenever $b_n(\delta)\le e^{-1}$. The theorem therefore
covers a fixed number of passes $K$ and does not show that further passes improve the direction.

\begin{theorem}[Recovery consistency]
\label{thm:recovery}
Let $\tilde b_k=b_n(\bar\delta_{k-1})$. Parts (a)--(b) hold under the hypotheses of Theorem~\ref{thm:joint}, on its
event $E_n$, when $\rho_0n\le|\mathcal F^*|\le(1-\rho_0)n$ with $\rho_0\in(0,\frac12]$. Part (c) holds under weaker
hypotheses, without Assumption~\ref{ass:linear}(iii)--(iv): Assumption~\ref{ass:linear}(i),(v), $\kappa_n\le n$,
Assumption~\ref{ass:endpoint}, $n\ge2p+2$, $\lambda_n\le w_{\mathrm{sat}}^2$, $\alpha_n\le\frac14$ (needed only
under $W_1$) and $b_n(\bar\delta_0)\le e^{-1}$, on $G_n\cap\{\delta_0\le\bar\delta_0\}$. For $1\le k\le K$:
\begin{enumerate}[label=(\alph*),leftmargin=*,nosep]
\item \emph{(flags, $w<1$)} clean comparisons with $|v_i|>\tilde b_k$ have $w_i^{(k)}=1$ and flips with
$|v_i|>\tilde b_k$ have $w_i^{(k)}\le\alpha_n<\tfrac12$; the share of flips left unflagged and the share of clean
comparisons flagged are each at most $(\bar f\tilde b_k+\eta_n)/\rho_0$.
\item \emph{(reversal, $w=0$)} flips with $|v_i|>\tilde b_k$ have $w_i^{(k)}=0$ under $W_\infty$, and so, for
$k\ge2$, do flips with $|v_i|>\max(\tilde b_k,\tilde b_{k-1})$ under $W_1$; the share of flips not reversed is at most
$(\bar f\tilde b_k+\eta_n)/\rho_0$ under $W_\infty$ and, for $k\ge2$, $(\bar f\max(\tilde b_k,\tilde b_{k-1})+\eta_n)/\rho_0$
under $W_1$.
\item \emph{(exact recovery; scope in Appendix~\ref{app:remarks})}
if every comparison is detectable, $\min_i|v_i|>b_n(\bar\delta_0)$, then $w^{(1)}=w^*$ under $W_\infty$ and
$\theta^{(1)}$ is an oracle (corrected-label) fit; under $W_1$ this holds at the second pass if also
$\min_i|v_i|>b_n(\delta_1)$.
\end{enumerate}
Under Assumption~\ref{ass:init} the bounds in (a)--(b) tend to zero: flagging and reversal are consistent.
\end{theorem}

\begin{corollary}[Rates under a small cut]
\label{cor:rates}
In Theorem~\ref{thm:joint} let $\bar\delta_0=\bar\delta_{0,n}$ be deterministic, fix $\zeta$ and $K$, put
$r_n=\log n/\kappa_n+\bar\delta_{0,n}+n^{-1/2}\log n$, and suppose $r_n\to0$ and $c_n\le Cr_n^2$. Then for large
$n$ the hypotheses of Theorem~\ref{thm:joint} hold and, on $E_n$, $\bar\delta_k=O(r_n\log^k(1/r_n))$ while the
weight error and the bounds of Theorem~\ref{thm:recovery}(a)--(b) are $O(r_n\log^{k-1}(1/r_n))$, $1\le k\le K$,
with constants depending only on $k,B,\bar f,c_*,C,C_p,\epsilon,\zeta,\rho_0$. The saturation weight enters
only through $w_{\mathrm{sat}}\ge(e^{c_n}-1)^{1/2}\asymp c_n^{1/2}$.
\end{corollary}
A larger cut, of order $\log n/\kappa_n+\delta_0$, keeps clean close calls at $w_i=1$ (Appendix~\ref{app:remarks}).

\begin{remark}[Reversal and down-weighting regimes]
\label{rem:regimes}
With a shrinking cut detected flips are \emph{reversed}: the regime of the guarantees.
At a fixed cut and a fixed agreement $\tilde m_i=-t$, the weight updates settle at an interior weight
$0<w_i<1$ exactly when $\lambda/(2+\lambda)<t\le\tau^{\mathrm{LLA}}$, a band whose upper end has a closed
form (Lemma~\ref{lem:lla-limit}); beyond it the weight reaches $0$. Such a
comparison is \emph{down-weighted} rather than reversed, the regime of the language-model runs. The guarantees concern the direction; the scale is identified only by close calls
(Remark~\ref{rem:scale}, Appendix~\ref{app:remarks}).
\end{remark}

\begin{proposition}[The $W_\infty$ step at a fixed cut]
\label{prop:fixedcut}
For the $W_\infty$ step, Theorem~\ref{thm:joint} and Theorem~\ref{thm:recovery}(a)--(c) remain true without
$\lambda_n\le w_{\mathrm{sat}}^2$ and without a bound on $\alpha_n$, provided condition \eqref{eq:Cinf} holds,
i.e.\ $\tau_n=\tau^{\mathrm{LLA}}(\lambda_n,a_n)<1$, with three changes: (i) $c_n^+$ is replaced in $b_n(\delta)$
by $c_n^\infty=2\,\mathrm{artanh}\,\tau_n$; (ii) $\alpha_n$ is replaced by $0$ in $F_n$ and in the weight-error
bound, and detectable flips receive weight exactly $0$; (iii) the factor $3B$ in $b_n(\delta)$ may be replaced by
any constant $\bar S$ with $\hat s(\theta^{(k)}/\|\theta^{(k)}\|)\le\bar S$ for $k<K$ ($\bar S=3B$ always works).
If $w_{\mathrm{sat}}$ is fixed and $\lambda_n<(1+\lambda_n)w_{\mathrm{sat}}$, then
$\tau_n\le\lambda_n/\{(1+\lambda_n)w_{\mathrm{sat}}\}$, so $c_n^\infty=O(c_n)$, against $c_n^+=O(c_n^{1/2})$.
\end{proposition}
Condition \eqref{eq:Cinf} is what makes reversal possible at all; $\lambda_n\le w_{\mathrm{sat}}^2$ and
$\alpha_n<\frac12$ serve only to make a detected flip's weight small within one or two single updates, as $W_1$
requires. Under $W_\infty$ the proposition covers a fixed cut with $w_{\mathrm{sat}}$ above the reversal threshold
(Table~\ref{tab:regimes}), including the flip-recovery configuration $c=1$, $w_{\mathrm{sat}}=0.3$
($c^\infty=1.87$), with non-asymptotic bounds; with $\bar S$ of the order of $\hat s$, a flip is guaranteed to be
reversed once its margin exceeds $c^\infty$ robust scales plus the direction error. Every configuration we report
($c\ge0.25$) violates $\lambda_n\le w_{\mathrm{sat}}^2$, so Theorems~\ref{thm:joint}--\ref{thm:recovery} as
stated describe the small-cut limit. The down-weighting configurations ($w_{\mathrm{sat}}=0.005$ at $c\ge0.5$,
used for the language model) also violate \eqref{eq:Cinf} and are not covered: their flagged comparisons keep
weights above $\frac12$ (Table~\ref{tab:regimes}), the counting step of Lemma~\ref{lem:thetastep} fails, and at
best a bias of the order of the flip rate can be expected.

\begin{figure}[t]
\centering
\includegraphics[width=\textwidth]{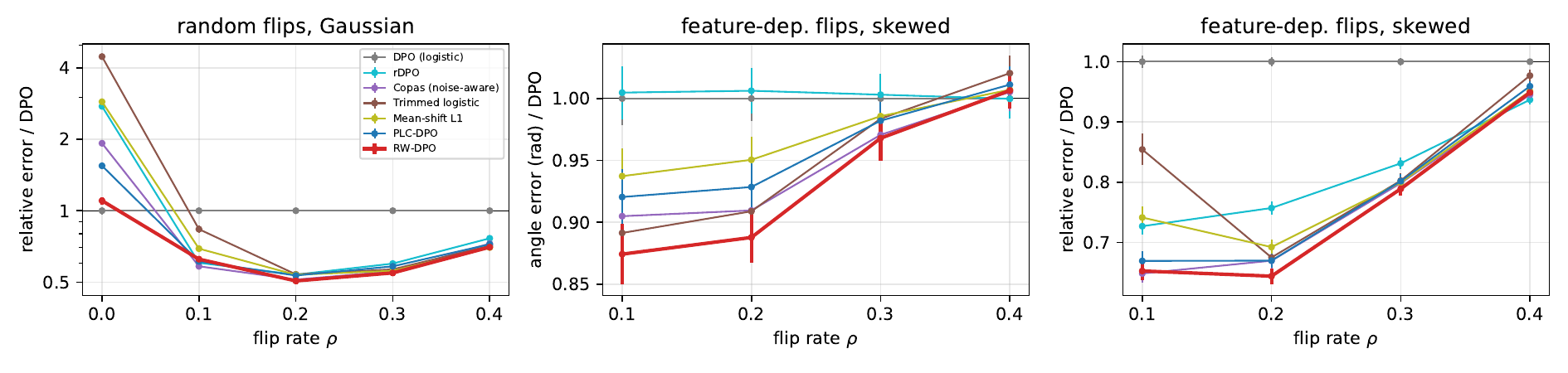}
\caption{Error relative to DPO (1 = no gain; bars $\pm2$ s.e.) against the flip rate, every method tuned on
clean validation data. Left: relative error under random flips (log scale). Middle and right: angle and
relative error under feature-dependent flips with skewed features. Settings in
Appendix~\ref{app:sim-details}.}
\label{fig:sim-main}
\end{figure}
\begin{table}[t]
\centering
\small
\caption{RW-DPO has the lowest relative error under both flip types and, on clear-cut flips, far more precise flags than PLC-DPO and mean-shift. Linear-index model, clean-validation tuning. Random flips and feature-dependent flips with skewed features: relative and angle error, means over $\rho\in\{0.1,\dots,0.4\}$, 12 settings, 100 replicates (s.e.\ $\le0.005$). Clear-cut flips (flips only where $|d_i^*|\ge3.5$, $\rho\in\{0.05,0.1\}$): detection, swamping and precision of the flags, and angle error. Designs, methods and metrics: Appendix~\ref{app:sim-details}. Results by flip rate: Table~\ref{tab:sim-full}. Bold: best non-oracle.}
\label{tab:sim-main}
\begin{tabular}{lcccccccc}
\toprule
 & \multicolumn{2}{c}{random flips} & \multicolumn{2}{c}{feature-dependent flips} & \multicolumn{4}{c}{clear-cut flips} \\
Method & rel.\ err. & angle & rel.\ err. & angle & detection & swamping & precision & angle \\
\midrule
Oracle labels & 0.162 & 0.150 & 0.167 & 0.153 & -- & -- & -- & 0.078 \\
\midrule
DPO (logistic) & 0.632 & 0.383 & 0.666 & 0.591 & -- & -- & -- & 0.150 \\
rDPO & 0.406 & 0.383 & 0.556 & 0.592 & -- & -- & -- & 0.147 \\
Copas (noise-aware) & 0.385 & \textbf{0.374} & 0.530 & 0.571 & -- & -- & -- & 0.093 \\
Trimmed logistic & 0.410 & 0.383 & 0.560 & 0.576 & 0.94 & 0.004 & 0.95 & 0.095 \\
Mean-shift $\ell_1$ & 0.393 & 0.377 & 0.545 & 0.580 & 0.99 & 0.345 & 0.19 & 0.121 \\
PLC-DPO & 0.394 & 0.377 & 0.536 & 0.577 & 0.99 & 0.151 & 0.35 & 0.111 \\
\textbf{RW-DPO} & \textbf{0.381} & 0.378 & \textbf{0.526} & \textbf{0.566} & 0.95 & 0.008 & 0.91 & \textbf{0.092} \\
\bottomrule
\end{tabular}
\end{table}

\section{Experiments}
\label{sec:exp}

RW-DPO is designed to keep the comparisons the model agrees with at full weight and to correct only those it
contradicts. We evaluate this selectivity in two settings: simulations of the linear-index model, in which the
flipped comparisons are known and the estimation error can be measured exactly, and DPO fine-tuning of a language
model, in which performance is measured by the quality of the trained policy.

\subsection{Simulations: accuracy and precision}
\label{sec:exp-linear}

In this section, we simulate the linear-index model of Section~\ref{sec:feasible}, where the flip set is known,
and test two aspects: the estimation error against existing robust methods, and how precisely the corrections
target the flips. Designs, methods, implementation, tuning and metrics are described in
Appendix~\ref{app:sim-details}; results by flip rate are in Table~\ref{tab:sim-full}.

\paragraph{RW-DPO estimates the preference most accurately (Table~\ref{tab:sim-main}, Figure~\ref{fig:sim-main}).}
Flips are drawn at random, or tied to a feature unrelated to the true preference so that they also bias the
direction DPO learns; each design covers $12$ settings, four flip rates and $100$ replicates. The competitors
represent the main ways of handling flips: correcting the loss for a global flip rate (rDPO and the
misclassification-robust likelihood of \citealt{copas1988}), discarding or absorbing suspicious comparisons
(trimmed logistic regression, an $\ell_1$ mean-shift fit), and routing each comparison by its margin (PLC-DPO).
With every method tuned on the same clean validation data, RW-DPO has the lowest relative error
$\|\hat\theta-\theta^*\|/\|\theta^*\|$ under both random flips ($0.381$, against $0.632$ for DPO) and
feature-dependent flips ($0.526$ against $0.666$), where it also has the lowest angle error ($0.566$). The differences
are small but systematic: on the same replicates RW-DPO's error is below every other method's by at least $4.9$
standard errors of the paired difference (Table~\ref{tab:sim-paired}). Figure~\ref{fig:sim-main} breaks the comparison
down by flip rate. RW-DPO is most accurate at the moderate flip rates $0.2$ and $0.3$, and, because it keeps
agreeing comparisons at full weight, it loses little accuracy when the labels are clean: with no flips its error is
$0.178$, against $0.162$ for DPO, whereas methods that assume a positive flip rate incur a larger error (Copas
$0.311$, rDPO $0.445$; Table~\ref{tab:sim-full}).

\paragraph{Its corrections are concentrated on the flips (Table~\ref{tab:sim-main}).} A low error could also
result from correcting many comparisons slightly. RW-DPO instead corrects few comparisons, predominantly flipped
ones. When flips occur only on clear-cut comparisons ($|d_i^*|\ge3.5$), as when annotators err on easy pairs,
RW-DPO flags $95\%$ of them
while flagging $0.8\%$ of the clean comparisons (\emph{swamping}), against $15\%$ for PLC-DPO and $35\%$ for
mean-shift; $91\%$ of its flags are true flips (\emph{precision}), against $0.35$ and $0.19$. PLC-DPO and
mean-shift attain slightly higher detection only at the cost of much higher swamping. The corrections do not
degrade the estimated direction: RW-DPO's angle error ($0.092$) is the lowest. As flips become clearer, its
detection rises from $0.70$ to $0.98$ and its precision from $0.83$ to $0.94$ (Table~\ref{tab:sim-full}),
consistent with the theory of Section~\ref{sec:theory}. (PLC-DPO makes no explicit flip decision; we count a
pair as flagged when ``flipped'' is the most probable state of its routing posterior,
Appendix~\ref{app:sim-details}.)

\paragraph{Summary.} In the simulations RW-DPO attains the lowest estimation error under both flip designs, with
little loss of accuracy on clean labels, and its corrections are concentrated on the flips, with substantially
higher precision than PLC-DPO and mean-shift and no loss in the estimated direction.

\begin{table}[t]
\centering
\small
\caption{RW-DPO matches PLC-DPO and stays above clean-label DPO at every flip rate, with its correction concentrated on the flips. Top: win rate (\%) against clean-label DPO (two epochs; mean $\pm$ s.d.\ over three seeds). Bottom: the correction of a pair is $1-\hat w_i$ for RW-DPO and the weight PLC-DPO's loss puts on the reversed label; ``applied'' is its mean over all pairs, ``on flips'' the share of the total correction that falls on flipped pairs ($\rho$ if spread at random; Table~\ref{tab:llm-landing}). Bold: best in each column.}
\label{tab:llm-summary}
\setlength{\tabcolsep}{3pt}
\begin{adjustbox}{max width=\columnwidth}\begin{tabular}{lccccc}
\toprule
Method & $\rho=0.0$ & $\rho=0.1$ & $\rho=0.2$ & $\rho=0.3$ & mean ($\rho>0$) \\
\midrule
DPO & 50.0\,{\scriptsize$\pm$1.1} & 45.9\,{\scriptsize$\pm$1.0} & 41.7\,{\scriptsize$\pm$1.8} & 34.5\,{\scriptsize$\pm$0.2} & 40.7 \\
rDPO ($\varepsilon=\rho$) & -- & 49.4\,{\scriptsize$\pm$0.6} & 46.6\,{\scriptsize$\pm$0.8} & 40.0\,{\scriptsize$\pm$1.3} & 45.3 \\
cDPO ($\varepsilon=\rho$) & -- & 43.9\,{\scriptsize$\pm$0.9} & 36.6\,{\scriptsize$\pm$0.8} & 29.9\,{\scriptsize$\pm$0.6} & 36.8 \\
PLC-DPO & \textbf{53.8}\,{\scriptsize$\pm$0.5} & \textbf{53.7}\,{\scriptsize$\pm$0.6} & \textbf{54.2}\,{\scriptsize$\pm$1.7} & 54.0\,{\scriptsize$\pm$1.8} & 54.0 \\
\textbf{RW-DPO} & 52.8\,{\scriptsize$\pm$0.2} & 53.4\,{\scriptsize$\pm$1.2} & 53.7\,{\scriptsize$\pm$1.1} & \textbf{55.1}\,{\scriptsize$\pm$2.9} & \textbf{54.1} \\
\midrule
\multicolumn{6}{l}{\emph{Correction (\%): applied, and share on flipped pairs}} \\
PLC-DPO applied & -- & 11 & 12 & 10 & 11 \\
RW-DPO applied & 3 & 3 & 4 & 5 & 4 \\
PLC-DPO on flips & -- & 10 & 20 & 30 & 20 \\
RW-DPO on flips & -- & \textbf{26} & \textbf{43} & \textbf{51} & \textbf{40} \\
\bottomrule
\end{tabular}\end{adjustbox}
\end{table}

\subsection{DPO fine-tuning of a language model}
\label{sec:exp-rq2}

In this section, we examine whether the gains seen in the simulations, where margins come from a correctly
specified model, carry over to DPO fine-tuning of a language model on randomly flipped preferences: we measure
how the win rate of the trained policy changes with the flip rate, and where each method's correction goes.

\paragraph{Setup.} We use the benchmark, recipe and evaluation code of PLC-DPO \citep{cho2026plc}:
\texttt{Qwen2.5-1.5B-}\allowbreak\texttt{ultrachat200k} with LoRA and the same model as frozen reference, $\beta=0.01$
(Appendix~\ref{app:llm}); $19{,}852$ UltraFeedback
comparisons with a share $\rho\in\{0,0.1,0.2,0.3\}$ of labels flipped at random, the range of PLC-DPO's benchmark,
three seeds per $\rho$ (flip masks used only for evaluation). RW-DPO uses $c=0.85$ and $w_{\mathrm{sat}}=0.005$: the
cut was chosen from $\{0.85,1,1.25\}$ by validation log-loss on one development cell ($\rho=0.2$, seed $0$;
Appendix~\ref{app:llm}) and held fixed for all cells, as is PLC-DPO's preset. Baselines are DPO, rDPO
\citep{chowdhury2024rdpo} and cDPO (label smoothing), both given the true flip rate, and PLC-DPO at its released
preset. In its authors' evaluation PLC-DPO had the highest mean win rate against DPO, ahead of cDPO, Dr.~DPO
\citep{wu2024drdpo}, ROPO \citep{liang2024ropo} and other robust variants, and the highest at every injected flip
rate of their noise test \citep{cho2026plc}. The metric is the win rate of greedy generations on $2000$ UltraFeedback prompts, judged by
Skywork-Reward-V2-Llama-3.1-8B (ties count one half), against clean-label DPO at the same training budget, as in
PLC-DPO (two epochs, averaged over three seeds; $50$ = as good as clean-label DPO).

\paragraph{RW-DPO holds its win rate as the flip rate grows (Table~\ref{tab:llm-summary}).} DPO degrades steadily as
the flip rate grows: its win rate against clean-label DPO falls from $45.9$ at $10\%$ flips to $34.5$ at $30\%$.
Correcting for the known flip rate offsets this only partly: rDPO recovers part of the loss and cDPO none (mean
$45.3$ and $36.8$, against $40.7$ for DPO). RW-DPO's win rate, in contrast, does not decrease with the flip rate:
it is $53.4$, $53.7$ and $55.1$ as the flip rate rises from $10\%$ to $30\%$, so the policy trained with RW-DPO
on flipped labels is preferred to DPO trained on clean labels at every flip rate. It matches PLC-DPO, the strongest recent method (mean $54.1$
against $54.0$), and with no flips both improve on DPO ($52.8$ and $53.8$). Clean preference accuracy falls
slightly for both (Appendix~\ref{app:llm}).

\paragraph{RW-DPO attains this with a lighter, more targeted correction.} With the flips known, we measure where each method's correction goes (Table~\ref{tab:llm-summary}, bottom; Table~\ref{tab:llm-landing}): $1-w_i$ for RW-DPO, and
for PLC-DPO the weight its loss puts on the reversed label, recomputed by its routing formula on the final model. Both are the weight a method places on the reversed label, in the likelihood for RW-DPO and in the loss for PLC-DPO, so their levels are comparable only roughly; the share on flipped pairs is scale-free. PLC-DPO applies about $11\%$, spread evenly: the share
that falls on flipped pairs equals the flip rate. RW-DPO applies $3$--$5\%$, of which $26$--$51\%$ falls on flipped
pairs, $1.7$--$2.6$ times chance, so a flipped pair loses $2.4$--$3.2$ times the weight of a clean one. Its flags
($w_i<1$) detect $30$--$31\%$ of the flips at $11$--$13\%$ swamping. Because a language model's margins identify
flips less sharply than the simulated ones (flag precision at most one half), RW-DPO down-weights suspected flips
rather than reversing them, a conservative response when flags are uncertain.

\paragraph{Re-weighting requires a trained policy.} The weights are computed from the policy's margins. After one
epoch RW-DPO ($c=0.5$) and PLC-DPO are comparable ($39.3$ and $39.2$) and
above the other methods ($29.3$--$30.4$); the second epoch sharpens the policy's margins and raises RW-DPO's flag precision from
$0.17$--$0.42$ to $0.24$--$0.50$.

\paragraph{Summary.} On the language model, RW-DPO trained on flipped labels yields a policy preferred to
clean-label DPO at every flip rate, matching the strongest recent method with corrections concentrated on the flips. In
both settings it keeps most comparisons at full weight and concentrates its correction on the comparisons most
likely to be mislabeled.

\section{Conclusion}

We proposed RW-DPO, which corrects preference data selectively: its penalized reliability weights keep every
comparison the model agrees with at exactly $w_i=1$ and adjust only those it contradicts. Theoretically, with a
fixed fraction of flipped labels the estimator is consistent in a linear-index model and recovers the flipped
comparisons above the $\log n$ margin scale, which limits what any method can recover from one annotation per
comparison. Empirically, it attains the lowest estimation error in simulations, with little loss of accuracy on
clean labels and substantially more precise flags than recent label-correction methods; in DPO fine-tuning with
$10$--$30\%$ flipped labels, its policy is preferred to DPO trained on clean labels and matches the strongest recent
method, with corrections concentrated on the flipped labels. The guarantees are established for a linear index with
binary true weights; ties and graded reliability are left for future work.

\subsection*{AI use statement}
We used generative AI tools, including coding assistants, to write and debug experiment code, run and
aggregate experiments, search and summarize related literature, assist with mathematical derivations, and
draft and edit the manuscript. The research question, methods, mathematical formulations, and final claims
were determined by the authors. All AI-assisted code was tested; reported results were verified against
experiment logs and code; mathematical derivations and proofs were independently checked by the authors;
citations were checked against their sources; and all text was reviewed by the authors. The authors take
responsibility for the final content, including the text, claims, and artifacts produced with the aid of
generative AI.

\clearpage
\appendix

\section{Model and algorithm details}
\label{app:details}

\subsection*{The $\theta$-step gradient}
With $w$ fixed, differentiating \eqref{eq:obj} gives
\begin{equation}
\label{eq:grad}
\nabla_\theta\Psi=-\sum_{i=1}^{n}\underbrace{\frac{w_i-\tfrac12}{r_i(\theta,w_i)}\;
g'\big(d_i(\theta)\big)\,Y_i}_{\textstyle c_i(\theta,w_i)}\;\nabla_\theta d_i(\theta),
\end{equation}
so the update is an ordinary DPO gradient step in which comparison $i$ carries the
scalar coefficient $c_i$. Three features are worth noting. At $w_i=1$, for the
Bradley--Terry link, the coefficient reduces to $Y_i\,\sigma(-Y_id_i)$, the standard DPO gradient
weight times the label sign, so the $\theta$-step is exactly DPO on the reliable comparisons. At
$w_i=\tfrac12$ the coefficient is zero and the comparison contributes nothing. At
$w_i<\tfrac12$ the coefficient changes sign, so a reversed comparison pushes the
policy in the corrected direction rather than merely being ignored.

\begin{remark}[Label-flip symmetry and initialization]
\label{rem:symmetry}
The objective \eqref{eq:obj} admits two
near-equivalent solutions related by wholesale relabeling. The per-comparison
reliability $r_i=\tfrac12+(w_i-\tfrac12)m_i$ is invariant under the joint map that
negates every reported label and complements every weight,
$(\{Y_i\},w)\mapsto(\{-Y_i\},\mathbf 1-w)$: negating $Y_i$ negates the agreement
$m_i$, while $w_i\mapsto1-w_i$ negates $(w_i-\tfrac12)$, so their product is
unchanged. Consequently a solution $(\theta^+,w\approx\mathbf1)$ that trusts the
reported labels and a solution $(\theta^-,w\approx\mathbf0)$ that trusts their
reverse attain the same likelihood, where $\theta^-$ is any policy inducing the
opposite margin $d_i(\theta^-)=-d_i(\theta^+)$ on every pair. In a linear reward
parametrization $d_i(\theta)=\theta^\top\phi_i$ this reversed policy is simply
$\theta^-=-\theta^+$; for a neural policy it is the opposite-ranking policy, not a
negation of the weights. It is the penalty, which charges $P(0)=a\lambda^2/2$ for
each reversed comparison and nothing for a reliable one, that breaks the tie: the
label-trusting basin reverses only the $s$ corrupted comparisons while its counterpart
reverses the other $n-s$, so the former is preferred by $(n-2s)\,a\lambda^2/2$ whenever the corrupted
comparisons are a minority, $s<n/2$. We
therefore initialize $\theta^{(0)}$ at a DPO fit (on the language model, a short DPO warm-up), which lies in the
label-trusting basin, to select it.
\end{remark}

\section{The weight step at a fixed margin}
\label{app:wstep}

We next study the scalar reliability problem that arises when $\theta$ is held
fixed. With $m=Yg(d(\theta))$ write
\[
F_{\lambda,a}(w;m)=-\log\Big[\tfrac12+\big(w-\tfrac12\big)m\Big]+P(w),\qquad
w\in[0,1].
\]

\begin{lemma}[Endpoint behaviour of the penalized weight]
\label{lem:endpoint}
For every $m\ge0$ the unique minimizer of $F_{\lambda,a}(w;m)$ is $w=1$. For $m<0$
the likelihood term is strictly increasing in $w$. Suppose $(\lambda,a)$ are such
that the endpoint region is nonempty, i.e.\ $w=0$ is the global minimizer for at
least one $m\in(-1,0)$ (the corner regime; guaranteed by
\eqref{eq:Cinf}, Lemma~\ref{lem:lla-limit}(e)). Then there is a threshold
$\tau^{\mathrm{exact}}(\lambda,a)\in(0,1)$ such that $w=0$ is the global minimizer
exactly when $m\le-\tau^{\mathrm{exact}}(\lambda,a)$.
\end{lemma}

\paragraph{Local linear approximation.}
The MCP derivative is $\varrho'(u;\lambda,a)=(\lambda-u/a)_+$; because $\varrho$
is concave, its tangent majorizes. At an iterate $w^{(t)}$ define
$\varpi^{(t)}=\varrho'(-\log w^{(t)};\lambda,a)=(\lambda+a^{-1}\log w^{(t)})_+$,
so the LLA surrogate is
$-\log[\tfrac12+(w-\tfrac12)m]+\varpi^{(t)}(-\log w)+\text{const}$, which is
minimized in closed form by Lemma~\ref{lem:wstep}.

The algorithm carries each weight across passes, so what matters at a fixed margin
is where the \emph{iterated} update ends. Fix $m=-t$ with $t\in(0,1)$ and write the
update of Lemma~\ref{lem:wstep} as a map of the current weight,
\[
T_t(w)=\begin{cases}
\min\Big\{1,\ \dfrac{\varpi(w)\,(1+t)}{2t\,(1+\varpi(w))}\Big\}, & \varpi(w)>0,\\[2mm]
0, & \varpi(w)=0,
\end{cases}
\qquad \varpi(w)=\big(\lambda+a^{-1}\log w\big)_+,
\]
so $\varpi(w)=0$ exactly when $w\le w_{\mathrm{sat}}=e^{-a\lambda}$.

\begin{lemma}[Warm-started LLA at a fixed margin]
\label{lem:lla-limit}
Let $w^{(0)}=1$ and $w^{(k+1)}=T_t(w^{(k)})$.
\begin{enumerate}[label=(\alph*),leftmargin=*,nosep]
\item The sequence is nonincreasing and converges to the largest fixed point
$w^{(\infty)}$ of $T_t$ in $[0,1]$.
\item $w^{(\infty)}=0$ if and only if $T_t(w)<w$ for every $w\in(w_{\mathrm{sat}},1]$;
in that case $w^{(k)}=0$ exactly for all large $k$. Otherwise $w^{(\infty)}$ is the
largest fixed point in $(w_{\mathrm{sat}},1]$, which equals $1$ when
$t\le\lambda/(2+\lambda)$.
\item Let $M(\lambda,a)=\max_{w\in[w_{\mathrm{sat}},1]}\varphi(w)/w$ with
$\varphi=\varpi/(1+\varpi)$. Then $w^{(\infty)}=0$ if and only if
$t>\tau^{\mathrm{LLA}}(\lambda,a)$, where
\[
\tau^{\mathrm{LLA}}=\frac{M}{2-M}\ \ \text{if }M<1,\qquad
\tau^{\mathrm{LLA}}=1\ \ \text{otherwise};
\]
$\tau^{\mathrm{LLA}}\ge\lambda/(2+\lambda)$, and at $t=\tau^{\mathrm{LLA}}$ the
comparison is attenuated ($w^{(\infty)}>0$).
\item $\tau^{\mathrm{LLA}}<1$ if and only if
\begin{equation}
\label{eq:Cinf}
\frac{\varpi(w)}{1+\varpi(w)}<w\qquad\text{for every }w\in[w_{\mathrm{sat}},1].
\tag{C$_\infty$}
\end{equation}
\item For $w\in(w_{\mathrm{sat}},1)$, $T_t(w)<w$ if and only if
$F'_{\lambda,a}(w;-t)>0$. Consequently, whenever $w^{(\infty)}=0$ the objective
$F_{\lambda,a}(\cdot\,;-t)$ is strictly increasing on $[0,1]$ and $w=0$ is its unique
global minimizer; in particular $\tau^{\mathrm{exact}}\le\tau^{\mathrm{LLA}}$, and
\eqref{eq:Cinf} implies that the endpoint region of Lemma~\ref{lem:endpoint} is
nonempty.
\end{enumerate}
For $m\ge0$ the update returns $w=1$ in one pass from any current weight
(Lemma~\ref{lem:wstep}).
\end{lemma}

So the iterated update keeps comparisons with $t\le\lambda/(2+\lambda)$, attenuates those
with $\lambda/(2+\lambda)<t<\tau^{\mathrm{LLA}}$ and reverses those with $t>\tau^{\mathrm{LLA}}$.

\section{The linear-index model: lemmas and remarks}
\label{app:remarks}

\subsection*{The two lemmas behind Theorem~\ref{thm:joint}}
\begin{lemma}[The weight step at any well-aligned fit]
\label{lem:wstep-uniform}
Under Assumptions~\ref{ass:linear}(i),(v) and~\ref{ass:endpoint}, let $n\ge2p+2$ (ensures $\hat s>0$, and $n\ge3$) and
$\lambda_n\le w_{\mathrm{sat}}^2$ (so $\lambda_n<1$ and $c_n^+$ is defined). On $G_n$ and almost surely,
simultaneously for every $\delta\ge0$, every $\theta\ne0$ with $\angle(\theta,\vartheta)\le\delta$, every weight vector $w$ and every
$i\in\mathcal D_n(\delta)$:
\begin{enumerate}[label=(\alph*),leftmargin=*,nosep]
\item if $w_i^*=1$, then $W_1(\theta,w)_i=W_\infty(\theta,w)_i=1$;
\item if $w_i^*=0$, then $W_1(\theta,w)_i\le\alpha_n$, with $W_1(\theta,w)_i=0$ when
$w_i\le w_{\mathrm{sat}}$; and $W_\infty(\theta,w)_i=0$.
\end{enumerate}
\end{lemma}

\begin{lemma}[The $\theta$-step with nearly correct weights]
\label{lem:thetastep}
Under Assumption~\ref{ass:endpoint}, on $G_n$, let $\delta\ge0$ be such that
$b=b_n(\delta)<1$, let $0\le\alpha<\tfrac12$, let $w\in[0,1]^n$ satisfy
$|w_i-w_i^*|\le\alpha$ for all $i\in\mathcal D_n(\delta)$. Then any $\hat\theta$ with
$L_n(\hat\theta;w)\le L_n(\theta_{\mathrm{ref}};w)$, $\theta_{\mathrm{ref}}=\{\log(1/b)/b\}\vartheta$ --- in particular
$\hat\theta(w)$ when $R_n\ge\log(1/b)/b$ --- satisfies
\[
\hat e(\hat\theta):=\frac1n\#\{i:\sign(\phi_i^\top\hat\theta)\ne\sign(v_i)\}
\le\gamma+\frac{b+\log\frac1{1-\alpha}+\gamma\{\log(1/b)+\log2\}}{\log 2},\qquad \gamma=\gamma_n(\delta).
\]
\end{lemma}

\subsection*{Scope of Theorem~\ref{thm:recovery}(c)}
Under Assumption~\ref{ass:linear}(iv), $\Prob(|v|\le b)\ge(c_*/B)\,b$ for $b\le B\pi$ (a sign disagreement at
angle $\delta$ forces $|v|\le B\delta$), so $\Prob\{\min_i|v_i|>b\}\le e^{-nc_*b/B}$ and the condition of part (c)
fails with high probability; conversely, a design with $|v|\ge v_0>0$ violates (iv), and then only part (c)
applies. Part (c) is therefore stated with its own, weaker hypotheses: Assumptions~\ref{ass:linear}(i),(v),
$\kappa_n\le n$ and~\ref{ass:endpoint}, with $n\ge2p+2$, $\lambda_n\le w_{\mathrm{sat}}^2$, $\alpha_n\le\frac14$
(needed only under $W_1$) and
$b_n(\bar\delta_0)\le e^{-1}$; it holds on $G_n\cap\{\delta_0\le\bar\delta_0\}$, of probability at least
$1-n^{-\epsilon}-\Prob(\delta_0>\bar\delta_0)$, and the oracle fit is any minimizer of the corrected-label
likelihood over the trust region. The choice of $\bar\delta_{0,n}$ and $\zeta_n$ in Theorem~\ref{thm:joint}:
$\bar\delta_{0,n}\to0$ with $\Prob(\delta_0>\bar\delta_{0,n})\to0$ (possible under Assumption~\ref{ass:init}) and
$\log(1/\zeta_n)=o(n/\log^2n)$.

\begin{remark}[The initial fit]
\label{rem:init}
The initial fit need only point in the right direction; its scale is irrelevant, since the weights use
standardized margins. Assumption~\ref{ass:init} is where the requirement that most comparisons be reliable enters, and it
is a genuine restriction on the contamination. When each comparison is flipped
independently of its features with probability $\rho<\tfrac12$ and the design is
rotationally symmetric, the reported label depends on $\phi$ only through
$\phi^\top\vartheta$, so the population minimizer of the logistic loss is proportional to
$\vartheta$ \citep{brillinger1982,liduan1989} with a positive constant; contamination shrinks
the scale of the DPO fit but not its direction, and plain DPO is expected to satisfy
Assumption~\ref{ass:init} (we expect the step from the population to the empirical minimizer to follow by
M-estimation arguments for fixed $p$; since $\kappa_n\to\infty$ it is not covered by standard results, and we do
not prove it). Feature-dependent flips can break it even when
they are a minority: if $\mathcal F^*$ is the wedge $\{\sign(\phi^\top\vartheta)\ne\sign(\phi^\top u')\}$
for a direction $u'\ne\vartheta$, the reported labels follow $\sign(\phi^\top u')$ wherever the clean label
equals $\sign(\phi^\top\vartheta)$, so DPO is pulled towards $u'$, and RW-DPO started there stays near it.
\end{remark}

\begin{remark}[Direction versus scale]
\label{rem:scale}
Theorem~\ref{thm:joint} is stated for the direction $\vartheta$, not for the relative error
$\|\hat\theta-\theta^*\|/\|\theta^*\|$, and a stronger statement should not be expected in general. The two
parts of $\theta^*=\kappa_n\vartheta$ are identified by different comparisons. Clear-cut comparisons
($|v_i|$ large) have labels that are almost surely correct whatever $\kappa_n$ is; they fix the
direction. Only close calls ($|v_i|$ small), whose labels err with probability $\sigma(-\kappa_n|v_i|)$,
carry information about the scale $\kappa_n$. A flip of a close call cannot be told apart from an
ordinary Bradley--Terry error (Theorem~\ref{thm:info-recovery}), so when flips occur among close
calls at a non-vanishing rate, they inflate the apparent error rate there and bias every estimator of
$\kappa_n$, ours included. The situation is different under \emph{clear-cut contamination},
$\mathcal F^*\subseteq\mathcal D_n(\delta)$: then every flip is reversed, at once under $W_\infty$ and from the second pass under $W_1$ (Lemma~\ref{lem:wstep-uniform}),
and if in addition the cut leaves the close calls at $w_i=1$, the $\theta$-step coincides with the
oracle fit on the clean labels and RW-DPO's first refit (second under $W_1$) inherits the
relative-error consistency of that fit (Corollary~\ref{cor:oracle} below). Both requirements on the cut --- large enough to leave close calls untouched, small enough that
$\lambda_n\le w_{\mathrm{sat}}^2$ --- can hold together only for large $n$, since $\log n/\kappa_n\to0$; at moderate
$n$ a cut that permits reversal also down-weights some clean close calls and biases the scale.
\end{remark}

\subsection*{Relative-error consistency under clear-cut contamination}

Remark~\ref{rem:scale} explains why the scale of $\theta^*$ cannot be guaranteed in general. When every
flip is clear-cut and the cut is slow enough to leave every clean comparison at $w_i=1$, the first
refit (the second under $W_1$) \emph{is} the oracle fit on the clean labels, and inherits whatever the
oracle fit achieves, relative error included. The one extra ingredient is a lower bound on the robust scale
$\hat s$, which needs the label-oriented margins to be spread out:
\begin{equation}
\label{eq:spread}
\text{there are }s_0>0,\ \eta_S>0,\ \delta_S>0\text{ with }
\sup_{q\in\Real}\Prob\big(|\phi^\top u-q|\le s_0\big)\le\tfrac14-\eta_S
\text{ for every unit }u\text{ with }\angle(u,\vartheta)\le\delta_S.
\tag{S}
\end{equation}
Condition \eqref{eq:spread} is not an extra assumption: under Assumption~\ref{ass:linear}(i),(v) it holds
with $\eta_S=\tfrac18$, $\delta_S=\pi$ and some $s_0>0$. Indeed $Q(u,s)=\sup_q\Prob(|\phi^\top u-q|\le s)\to0$ as
$s\downarrow0$ for each unit $u$, since $\phi^\top u$ has a Lebesgue density, and
$Q(u,s)\le Q(u',s+B\|u-u'\|)$; covering the unit sphere by finitely many balls of radius $s'(u')/B$ with
$Q(u',2s'(u'))<\tfrac18$ gives $Q(u,s_0)<\tfrac18$ for all $u$, with $s_0=\min s'$. For a design uniform on a
ball one may take any $s_0<1/(8\sup_u\|f_u\|_\infty)$, where $f_u$ is the density of $\phi^\top u$.

\begin{corollary}[Oracle property under clear-cut contamination]
\label{cor:oracle}
Let Assumptions~\ref{ass:linear}(i),(v) and~\ref{ass:endpoint} hold, and let $s_0$ be a constant for which
\eqref{eq:spread} holds; let
$c_n\to0$ and $\delta_0=\angle(\theta^{(0)},\vartheta)$, and suppose
\begin{enumerate}[label=(\roman*),leftmargin=*,nosep]
\item \emph{(clear-cut contamination)} $\Prob\{\mathcal F^*\subseteq\mathcal D_n(\delta_0)\}\to1$;
\item \emph{(slow cut)} $\Prob\big\{(1+\epsilon)\log n/\kappa_n+B\delta_0\le s_0c_n\big\}\to1$.
\end{enumerate}
Let $\ell_n(\theta)=-\frac1n\sum_{i=1}^n\log\sigma(\widetilde Y_i\phi_i^\top\theta)$ be the oracle
(corrected-label) negative log-likelihood and $\Theta^{\mathrm{or}}_n=\argmin_{\|\theta\|\le R_n}\ell_n$ the set
of trust-region oracle fits. With probability tending to one:
\begin{enumerate}[label=(\alph*),leftmargin=*,nosep]
\item under $W_\infty$, $w^{(1)}=w^*$ exactly and $\theta^{(1)}\in\Theta^{\mathrm{or}}_n$;
\item under $W_1$, $w^{(1)}_i=1$ on every clean comparison and $w^{(1)}_i\le\alpha_n$ on every flip; if $\theta^{(1)}\ne0$ and (i)
and (ii) hold also with $\delta_1=\angle(\theta^{(1)},\vartheta)$ in place of $\delta_0$ --- conditions on the first
refit that the other hypotheses do not guarantee --- then $w^{(2)}=w^*$ and
$\theta^{(2)}\in\Theta^{\mathrm{or}}_n$.
\end{enumerate}
Consequently, if the oracle fit is relative-error consistent,
$\sup_{\theta\in\Theta^{\mathrm{or}}_n}\|\theta-\theta^*_n\|/\|\theta^*_n\|\xrightarrow{p}0$, then so is RW-DPO:
$\|\theta^{(1)}-\theta^*_n\|/\|\theta^*_n\|\xrightarrow{p}0$ under $W_\infty$, and, under the added hypotheses of (b), the same for $\theta^{(2)}$ under $W_1$.
\end{corollary}

Condition (i) asks that no flip lie in the band $|v_i|\le b_n(\delta_0)$, which shrinks to zero; flips at
$|v_i|\ge t_0$ for a fixed $t_0>0$ qualify (the contamination pattern of the clear-cut design of Section~\ref{sec:exp-linear}, although that design, with a fixed
signal and Gaussian features, is not itself within the corollary's asymptotics). Condition (ii) asks that the cut be at least a constant multiple of
both the information scale $\log n/\kappa_n$ and the initial direction error; it implies
Assumption~\ref{ass:init}, and it is what keeps the close calls --- the only comparisons that carry
information about $\kappa_n$ --- at full weight. The proof also uses $\lambda_n\le w_{\mathrm{sat}}^2$
(Lemma~\ref{lem:wstep-uniform}), which holds eventually because $c_n\to0$. At a given $n$ the corollary's
two requirements on the cut, $c_n\ge\{(1+\epsilon)\log n/\kappa_n+B\delta_0\}/s_0$ and
$c_n\le\log(1+w_{\mathrm{sat}}^2)$, can hold together only when
$(1+\epsilon)\log n/\kappa_n\le s_0\log(1+w_{\mathrm{sat}}^2)\approx s_0w_{\mathrm{sat}}^2$ (sufficient conditions
of the proof, not necessary ones for the estimator). The corollary is conditional on the oracle: maximum likelihood with a growing
signal learns $\kappa_n$ only from the close calls, heuristically a fraction of order $1/\kappa_n$ of the
sample, so its relative-error consistency needs $\kappa_n$ to grow well below $n$; for standard Gaussian
design, \citet{hsumazumdar2024} show how the sample complexity of estimating the logistic parameter to a
given $\ell_2$ error depends on the dimension and on the inverse temperature (here $\kappa_n$), with distinct
high-, moderate- and low-temperature regimes; their setting is not covered by
Assumption~\ref{ass:linear}(i), and we do not use their results.

\begin{remark}[The cut and clean close calls]
\label{rem:cut-tradeoff}
A cut with $c_n=O(r_n^2)$ (Corollary~\ref{cor:rates}) flags clean close calls, counted as swamping in
Theorem~\ref{thm:recovery}(a). Corollary~\ref{cor:oracle} gives sufficient conditions for keeping every
clean comparison at $w_i=1$ and reversing every flip: $c_n\ge\{(1+\epsilon)\log n/\kappa_n+B\delta_0\}/s_0$
with probability tending to one, $w_{\mathrm{sat}}\ge(e^{c_n}-1)^{1/2}$, and every flip in
$\mathcal D_n(\delta_0)$. Since $c_n^+\le4\alpha_n\le4(2c_n)^{1/2}$ for $c_n\le1$ and $\alpha_n\le\frac12$, the last holds
when every flip has $|v_i|>\max\{(1+\epsilon)\log n/\kappa_n,\ B\delta_0+12B(2c_n)^{1/2}\}$; then, with
probability tending to one, $w^{(1)}=w^*$ and $\theta^{(1)}\in\Theta^{\mathrm{or}}_n$ under $W_\infty$ (under $W_1$,
at the second pass, with the added hypotheses of Corollary~\ref{cor:oracle}(b)). This flip separation is
of order $c_n^{1/2}$ and grows with the cut, so the sufficient conditions are weakest for a cut of order
$\log n/\kappa_n+\delta_0$, where the separation is of order $(\log n/\kappa_n+\delta_0)^{1/2}$ --- larger than
the separation $\log n/\kappa_n+\delta_0+c_n^{1/2}$ needed for (a)--(b) at a cut $c_n=O(r_n^2)$. These are
sufficient conditions; in practice the cut is a tuning constant chosen on validation data.
\end{remark}

\begin{remark}[Why the endpoint assumption is needed]
\label{rem:endpoint}
The endpoint assumption converts reliability estimation into flip-support
recovery. When $w_i^*\in(0,1)$, one binary observation is a single Bernoulli draw
whose success probability depends on $w_i^*$, so individual consistency of an
arbitrary sequence of interior weights cannot be obtained without repeated
annotations or further structure. RW-DPO nonetheless continues to allow
$w_i\in(0,1)$ and uses such values for graded attenuation in finite samples.
\end{remark}

\section{Proofs}
\label{app:proofs}

\begin{proof}[Proof of Lemma~\ref{lem:reduce}]
$w\tfrac{1+m}{2}+(1-w)\tfrac{1-m}{2}=\tfrac12+(w-\tfrac12)m$.
\end{proof}

\begin{proof}[Proof of Lemma~\ref{lem:wstep}]
Write $r(w)=\tfrac12+(w-\tfrac12)m$, so $r'(w)=m$ and
$h'(w)=-m/r(w)-\varpi/w$. If $m>0$, or $m=0<\varpi$, then $h'<0$ on $(0,1]$ and the
minimum is at $w=1$; if $m=0=\varpi$ then $h\equiv\log2$ and $\hat w=1$ by convention. If $m<0$, setting $h'(w)=0$ gives $-mw=\varpi r(w)$, i.e.
$-mw=\varpi[\tfrac12+(w-\tfrac12)m]$; solving,
$-wm(1+\varpi)=\varpi(1-m)/2$ and hence
$w=\varpi(1-m)/\{2|m|(1+\varpi)\}$. The stationary point is a minimum since
$h''>0$ there, and clipping to $[0,1]$ handles the boundary. Setting the
expression $\ge1$ and rearranging gives $m\ge-\varpi/(2+\varpi)$. If $\varpi=0$
the penalty vanishes and $h'=-m/r>0$ for $m<0$, so the minimum is $w=0$.
\end{proof}

\begin{proof}[Proof of Proposition~\ref{prop:bayes}]
Under $H_{1i}$ and $H_{0i}$ the two Bernoulli distributions have success
probabilities $p_{1i}=\tfrac{1+g_i^*}{2}$ and $p_{0i}=\tfrac{1-g_i^*}{2}$. The
Bayes error for two simple hypotheses with equal priors is
$\tfrac12(1-\TV(P_{1i},P_{0i}))$, and for Bernoulli laws
$\TV(P_{1i},P_{0i})=|p_{1i}-p_{0i}|=|g_i^*|$. Hence $e_i^*=\tfrac{1-|g_i^*|}{2}$.
For Bradley--Terry, $\tfrac{1-\tanh(|d|/2)}{2}=\tfrac{1}{1+e^{|d|}}$, which proves
\eqref{eq:bayes-error-bt}.
\end{proof}

\begin{proof}[Proof of Theorem~\ref{thm:info-recovery}]
For Bradley--Terry, Proposition~\ref{prop:bayes} gives
$e_i^*\le e^{-|d_i^*|}\le n^{-(1+\epsilon)}$, so by the union bound
$\Prob(\widehat{\mathcal F}_{\mathrm{Bayes}}\ne\mathcal F^*)\le\sum_i e_i^*\le
n^{-\epsilon}\to0$; this holds for every fixed $\mathcal F^*$, and the equal prior is used only in the converse,
where a lower bound on the Bayes risk is also a lower bound on the worst case over $\mathcal F^*$. For the converse, the estimator maximizing
$\Prob(\widehat{\mathcal F}=\mathcal F^*)$ is the joint posterior mode. Under the
product prior and conditional independence the posterior of $w^*$ factorizes over
comparisons, so the joint mode is the vector of coordinatewise Bayes tests and its
success probability is $\prod_i(1-e_i^*)$. If $|d_i^*|\le M$ on $\mathcal I_n$
then $e_i^*\ge\tfrac{1}{1+e^M}=:c_0>0$ for each such $i$, so every estimator has
$\Prob(\widehat{\mathcal F}=\mathcal F^*)\le(1-c_0)^{c_In}\to0$.
\end{proof}

\begin{proof}[Proof of Lemma~\ref{lem:endpoint}]
For $m\ge0$, $\partial_w\{-\log[\tfrac12+(w-\tfrac12)m]\}=-m/[\tfrac12+(w-\tfrac12)m]\le0$,
so the likelihood term decreases in $w$; since $P\ge0$ and $P(1)=0$, both terms are
minimized at $w=1$. For $m<0$ write $m=-t$, $t\in(0,1)$, and
$A(w,t)=\tfrac12+(w-\tfrac12)(-t)=\tfrac{1+t}{2}-wt>0$, so
$F_{\lambda,a}(w;-t)=-\log A(w,t)+P(w)$. Let $\Delta(w,t)=F_{\lambda,a}(0;-t)-F_{\lambda,a}(w;-t)$;
the corner $w=0$ is the global minimizer iff $\Delta(w,t)\le0$ for all $w\in(0,1]$.
A direct computation gives, for each fixed $w>0$,
\[
\frac{\partial\Delta}{\partial t}
=\frac{\tfrac12-w}{A(w,t)}-\frac{1}{1+t}
=\frac{(\tfrac12-w)(1+t)-A(w,t)}{A(w,t)(1+t)}
=\frac{-w}{A(w,t)(1+t)}<0 ,
\]
so $\Delta(w,\cdot)$ is strictly decreasing in $t$ for every $w>0$. Hence if the
corner is global at some $t_0$ (i.e.\ $\Delta(\cdot,t_0)\le0$ everywhere), it is
global at every $t\ge t_0$; the set $S=\{t\in(0,1):\ w=0\text{ global}\}$ is therefore an
upper interval, nonempty by hypothesis. It stays away from $0$: as $t\downarrow0$,
$\Delta(1,t)\to P(0)-P(1)=P(0)>0$, so $\Delta(1,t)>0$ for all small $t$. It is closed in $(0,1)$,
being an intersection of the sets $\{t:\Delta(w,t)\le0\}$, each closed by continuity of
$\Delta(w,\cdot)$. Hence its infimum $\tau^{\mathrm{exact}}\in(0,1)$ belongs to it, and
$S=[\tau^{\mathrm{exact}},1)$.
\end{proof}

\begin{proof}[Proof of Lemma~\ref{lem:lla-limit}]
\emph{Monotonicity.} $\varpi$ is nondecreasing in $w$, $v\mapsto v/(1+v)$ is
increasing, and $\min\{1,\cdot\}$ preserves order; as $\varpi(w)\downarrow0$ the
first branch tends to $0$, so $T_t$ is nondecreasing and continuous on $[0,1]$
(with $T_t=0$ on $[0,w_{\mathrm{sat}}]$). For fixed $w$, $T_t(w)$ is nonincreasing
in $t$ because $(1+t)/t$ is decreasing.

(a) $w^{(1)}=T_t(1)\le1=w^{(0)}$, and if $w^{(k)}\le w^{(k-1)}$ then
$w^{(k+1)}=T_t(w^{(k)})\le T_t(w^{(k-1)})=w^{(k)}$; so the sequence is
nonincreasing, bounded below by $0$, and converges to some $w^{(\infty)}$, which is
a fixed point by continuity of $T_t$. If $\bar w$ is any fixed point then
$w^{(0)}=1\ge\bar w$, and $w^{(k)}\ge\bar w$ implies
$w^{(k+1)}=T_t(w^{(k)})\ge T_t(\bar w)=\bar w$; hence $w^{(\infty)}\ge\bar w$.

(b) $T_t=0$ on $(0,w_{\mathrm{sat}}]$, so $T_t$ has no fixed point there. If
$T_t(w)<w$ on $(w_{\mathrm{sat}},1]$, the only fixed point is $0$ and
$w^{(\infty)}=0$ by (a); since $w^{(k)}\downarrow0<w_{\mathrm{sat}}$, some
$w^{(k)}\le w_{\mathrm{sat}}$, and then $w^{(k+1)}=0$, which is absorbing. If instead
$T_t(\bar w)\ge\bar w$ for some $\bar w\in(w_{\mathrm{sat}},1]$, then $h=T_t-\mathrm{id}$
satisfies $h(\bar w)\ge0\ge h(1)$, so by continuity there is a fixed point in
$[\bar w,1]$ and $w^{(\infty)}>0$ by (a). When $t\le\lambda/(2+\lambda)$,
$T_t(1)=1$ by Lemma~\ref{lem:wstep}, so $w^{(\infty)}=1$.

(c) For $w\in(w_{\mathrm{sat}},1]$, $T_t(w)<w$ holds exactly when the unclipped
value is below $w$ (a clipped value equals $1\ge w$), i.e.\ when
$\varphi(w)(1+t)/(2t)<w$, i.e.\ $\varphi(w)/w<2t/(1+t)$; and $\varphi(w_{\mathrm{sat}})=0$.
So by (b), $w^{(\infty)}=0$ iff $\varphi(w)/w<2t/(1+t)$ on $[w_{\mathrm{sat}},1]$,
iff $M<2t/(1+t)$ (the maximum is attained, $\varphi$ being continuous on a compact
set), iff $t>M/(2-M)$ when $M<1$; when $M\ge1$ no $t<1$ qualifies. At
$t=\tau^{\mathrm{LLA}}$ the maximizer $\bar w$ has $T_t(\bar w)=\bar w$, so
$w^{(\infty)}>0$ by (b). The lower bound follows from the last statement of (b).

(d) As $t\uparrow1$, $T_t(w)\downarrow T_1(w):=\varpi(w)/(1+\varpi(w))$, with
$0\le T_t(w)-T_1(w)\le(1-t)/(2t)$ uniformly in $w$. If \eqref{eq:Cinf} holds,
$w-T_1(w)$ is continuous and positive on the compact $[w_{\mathrm{sat}},1]$, hence at
least some $\nu>0$ there; for $t$ with $(1-t)/(2t)<\nu$ we get $T_t(w)<w$ on
$(w_{\mathrm{sat}},1]$, so $\tau^{\mathrm{LLA}}<1$. Conversely, if
$T_1(\bar w)\ge\bar w$ for some $\bar w\in[w_{\mathrm{sat}},1]$ (necessarily
$\bar w>w_{\mathrm{sat}}$, as $T_1(w_{\mathrm{sat}})=0$), then
$T_t(\bar w)\ge T_1(\bar w)\ge\bar w$ for every $t<1$, and (b) gives
$w^{(\infty)}>0$ for every $t<1$, so $\tau^{\mathrm{LLA}}=1$. Finally, if
$(1+\lambda)e^{-a\lambda}>\lambda$, i.e.\ $T_1(1)=\lambda/(1+\lambda)<w_{\mathrm{sat}}$,
then $T_1(w)\le T_1(1)<w_{\mathrm{sat}}\le w$ on $[w_{\mathrm{sat}},1]$, which is
\eqref{eq:Cinf}. (Equivalently, \eqref{eq:Cinf} is $M<1$.)

(e) Fix $w\in(w_{\mathrm{sat}},1)$, so $\varpi(w)>0$. The LLA surrogate
$h_w(v)=-\log[\tfrac{1+t}{2}-vt]+\varpi(w)(-\log v)$ is strictly convex in $v$, and
$T_t(w)$ is its minimizer over $[0,1]$ (Lemma~\ref{lem:wstep}). Hence $T_t(w)<w$ iff
$h_w'(w)>0$. Since $P'(w)=-\varrho'(-\log w)/w=-\varpi(w)/w$,
\[
h_w'(w)=\frac{t}{\tfrac{1+t}{2}-wt}-\frac{\varpi(w)}{w}=F'_{\lambda,a}(w;-t).
\]
On $(0,w_{\mathrm{sat}})$, $P$ is constant, so $F'=t/(\tfrac{1+t}{2}-wt)>0$; $P$ is
$C^1$ on $(0,1)$ because $\varpi\to0$ at $w_{\mathrm{sat}}$; and at $w=1$,
$T_t(1)<1$ iff $F'(1^-)>0$ by the same argument. So if $w^{(\infty)}=0$, i.e.\
$T_t(w)<w$ on $(w_{\mathrm{sat}},1]$ by (b), then $F'>0$ on $(0,1)$, $F$ is strictly
increasing on $[0,1]$, and $w=0$ is its unique minimizer. Every
$t>\tau^{\mathrm{LLA}}$ therefore lies in the endpoint region of
Lemma~\ref{lem:endpoint}, which gives $\tau^{\mathrm{exact}}\le\tau^{\mathrm{LLA}}$
and, under \eqref{eq:Cinf}, nonemptiness of that region.
\end{proof}

\begin{proof}[Proof of Lemma~\ref{lem:wstep-uniform}]
Fix such a $\theta$ and let $u=\theta/\|\theta\|$; $z_i$ depends on $\theta$ only
through $u$, because $\xi_i$ and $\hat s$ are both proportional to $\|\theta\|$. Since
$\|u-\vartheta\|=2\sin(\angle/2)\le\delta$,
$|\phi_i^\top u-v_i|\le B\delta$ for every $i$. Each $|\xi_i(u)|\le B$, so
$|\xi_i-\mathrm{med}\,\xi|\le2B$ and $\hat s(u)\le2.966B<3B$. Positivity holds for all $u$ at once,
almost surely: $\hat s(u)=0$ would put at least $n/2\ge p+1$ of the points $Y_i\phi_i$ on one
affine hyperplane $\{x:x^\top u=\mathrm{med}\,\xi\}$, and under a density (Assumption~\ref{ass:linear}(v))
no $p+1$ of the $\pm\phi_i$ are affinely dependent, almost surely. (Only the upper bound
on $\hat s$ is used below.) Take $i\in\mathcal D_n(\delta)$; then
$\kappa_n|v_i|>(1+\epsilon)\log n$, so on $G_n$, $\widetilde Y_i=\sign(v_i)$.
(a) $Y_i=\sign(v_i)$, so $\xi_i(u)=|v_i|+\sign(v_i)\phi_i^\top(u-\vartheta)\ge|v_i|-B\delta>0$;
then $\tilde m_i>0$, and Lemma~\ref{lem:wstep} returns $w=1$ from any current weight,
which is then a fixed point.
(b) $Y_i=-\sign(v_i)$, so $\xi_i(u)\le-(|v_i|-B\delta)<-3Bc_n^+$, hence $z_i<-c_n^+$ and
$t_i:=|\tilde m_i|>\tanh(c_n^+/2)=\alpha_n$. Whatever the current weight, the adaptive
weight satisfies $0\le\varpi\le\lambda_n$, and Lemma~\ref{lem:wstep} gives
\[
W_1(\theta,w)_i\le\frac{\varpi(1+t_i)}{2t_i(1+\varpi)}\le\frac{\lambda_n}{t_i}<\frac{\lambda_n}{\alpha_n}=\alpha_n .
\]
If $w_i\le w_{\mathrm{sat}}$ then $\varpi=0$ and the update is $0$. Since
$\alpha_n\le w_{\mathrm{sat}}$, a second pass at the same $\theta$ therefore returns $0$,
which is absorbing; so the repeated updates reach $0$ within two passes from any start,
and $W_\infty(\theta,w)_i=0$.
\end{proof}

\begin{proof}[Proof of Lemma~\ref{lem:thetastep}]
\emph{Pairs with nearly correct weights.} Let $i\in\mathcal D_n(\delta)$,
$\alpha_i=|w_i-w_i^*|$ and $s_i=\widetilde Y_i\phi_i^\top\theta$. If $w_i^*=1$ then
$Y_i=\widetilde Y_i$ and $w_i=1-\alpha_i$; if $w_i^*=0$ then $Y_i=-\widetilde Y_i$ and
$w_i=\alpha_i$. In both cases, using $\tfrac12+\tfrac12\tanh(s/2)=\sigma(s)$,
\[
r_i=\tfrac12+\big(w_i-\tfrac12\big)\tanh\tfrac{Y_i\phi_i^\top\theta}{2}=(1-\alpha_i)\,\sigma(s_i)+\alpha_i\,\sigma(-s_i);
\]
with $\alpha_i=0$ this is the corrected-label likelihood.
\emph{Any weight.} Since $\tfrac12+(w-\tfrac12)\tanh(t/2)\ge\tfrac12-\tfrac12|\tanh(t/2)|=\sigma(-|t|)$,
every term satisfies $0\le-\log r_i\le\log(1+e^{|\phi_i^\top\theta|})\le|\phi_i^\top\theta|+\log2$.
\emph{Upper bound.} By hypothesis $L_n(\hat\theta;w)\le L_n(\theta_{\mathrm{ref}};w)$, with
$\theta_{\mathrm{ref}}=\{\log(1/b)/b\}\vartheta$ (for $\hat\theta(w)$ this holds because $\theta_{\mathrm{ref}}$ is feasible,
$\|\theta_{\mathrm{ref}}\|\le R_n$). For $i\in\mathcal D_n(\delta)$, on $G_n$,
$s_i=|v_i|\log(1/b)/b\ge\log(1/b)$, so $r_i\ge(1-\alpha)\sigma(s_i)$ and the term is at most
$\log\frac1{1-\alpha}+\log(1+b)\le\log\frac1{1-\alpha}+b$. For $i\notin\mathcal D_n(\delta)$,
$|v_i|\le b$ and the term is at most $\log(1/b)+\log2$. Hence
$L_n(\hat\theta;w)\le L_n(\theta_{\mathrm{ref}};w)\le b+\log\frac1{1-\alpha}+\gamma\{\log(1/b)+\log2\}$.
\emph{Lower bound.} If $i\in\mathcal D_n(\delta)$ and $\sign(\phi_i^\top\hat\theta)\ne\sign(v_i)=\widetilde Y_i$,
then $s_i\le0$, so $r_i=\alpha_i+(1-2\alpha_i)\sigma(s_i)\le\tfrac12$ (as $\alpha_i\le\alpha<\tfrac12$) and the term
is at least $\log2$; all terms are nonnegative. So the share of
$\mathcal D_n(\delta)$-pairs misclassified by $\hat\theta$ is at most
$L_n(\hat\theta;w)/\log2$; the pairs outside $\mathcal D_n(\delta)$ add at most $\gamma$.
\end{proof}

\begin{proof}[Proof of Theorem~\ref{thm:joint}]
\emph{Good event.} Let $H_n=\{\sup_{u\ne0}|\hat e(u)-e(u)|\le\eta_n\}$ and
$H_n'=\{\sup_{b\ge0}|\tfrac1n\#\{i:|v_i|\le b\}-\Prob(|v|\le b)|\le\eta_n\}$. The sets
$\{x:\sign(x^\top u)\ne\sign(x^\top\vartheta)\}$, $u\ne0$, are Boolean combinations of a fixed number of
half-spaces through the origin and fixed sets, so they form a VC class with dimension depending only on $p$.
By the chaining bound for VC classes \citep[Thm.~2.6.7 and Sec.~2.14]{vdvw1996} the expected supremum in $H_n$
is at most $C_pn^{-1/2}$, and by the bounded-difference inequality $\Prob(H_n^c)\le\zeta$; the
Dvoretzky--Kiefer--Wolfowitz inequality (with Massart's constant) gives $\Prob(H_n'^c)\le\zeta$.
With $\Prob(G_n^c)\le n^{-\epsilon}$, the event $E_n=G_n\cap H_n\cap H_n'\cap\{\delta_0\le\bar\delta_0\}$ has
$\Prob(E_n)\ge1-n^{-\epsilon}-2\zeta-\Prob(\delta_0>\bar\delta_0)$.

\emph{The one-step bound.} For $\delta$ with $b=b_n(\delta)\le\min\{e^{-1},b_0\}$ put
$\gamma^+=\bar f\,b+\eta_n$ and
\[
\bar F_n(\delta)=c_*^{-1}\Big[\eta_n+\gamma^++\frac{b+\log\frac1{1-\alpha_n}+\bar f\,b\{\log(1/b)+\log2\}+\eta_n(\log n+\log2)}{\log2}\Big].
\]
Multiplying out, $c_*\log2\cdot\bar F_n(\delta)=b+\bar f\,b\log(4/b)+\log\frac1{1-\alpha_n}+\eta_n\log(8n)$. Since
$\log(4/b)\ge1$ and $\log\frac1{1-\alpha}\le2\alpha$ for $\alpha\le\frac12$, $\bar F_n\le F_n$. Moreover $F_n$ is
nondecreasing on $\{\delta:b_n(\delta)\le e^{-1}\}$, because $b_n$ is nondecreasing in $\delta$ and
$b\mapsto b\log(4/b)$ is increasing on $(0,4/e)$.

\emph{Induction on $E_n$.} Suppose $\delta_{k-1}\le\bar\delta_{k-1}$ and put $b=b_n(\delta_{k-1})\le b_n(\bar\delta_{k-1})$.
Lemma~\ref{lem:wstep-uniform}, applied at $\theta=\theta^{(k-1)}$ with current weights $w^{(k-1)}$, gives
$|w_i^{(k)}-w_i^*|\le\alpha_n$ on $\mathcal D_n(\delta_{k-1})$. On $H_n'$ and by
Assumption~\ref{ass:linear}(iii), $\gamma_n(\delta_{k-1})\le\bar f b+\eta_n=\gamma^+$.
Lemma~\ref{lem:thetastep} applies with $\alpha=\alpha_n$, since $b\ge(1+\epsilon)\log n/\kappa_n$ and
$\kappa_n\le n$ give $\log(1/b)\le\log n$ and $\log(1/b)/b\le\kappa_n\le R_n$; using $\log(1/b)\le\log n$ for the
$\eta_n$ part of $\gamma^+$, it gives $\hat e(\theta^{(k)})\le c_*\bar F_n(\delta_{k-1})-\eta_n$. (Here $\theta^{(k)}\ne0$: the bound of Lemma~\ref{lem:thetastep} has the form $\gamma+U/\log2$ with
$U\ge L_n(\theta_{\mathrm{ref}};w^{(k)})$, and it is at most $c_*\bar\delta_k-\eta_n<1$, since $\bar\delta_k<\pi$ and
$c_*\le1/\pi$ (take $u=-\vartheta$ in Assumption~\ref{ass:linear}(iv)); so
$L_n(\theta^{(k)};w^{(k)})\le L_n(\theta_{\mathrm{ref}};w^{(k)})<\log2=L_n(0;w^{(k)})$.) On $H_n$, $e(\theta^{(k)})\le c_*\bar F_n(\delta_{k-1})$, and
Assumption~\ref{ass:linear}(iv) gives $\delta_k\le\bar F_n(\delta_{k-1})\le F_n(\delta_{k-1})\le F_n(\bar\delta_{k-1})=\bar\delta_k$.
The weight bound follows: $\frac1n\sum_i|w_i^{(k)}-w_i^*|\le\gamma_n(\delta_{k-1})+\alpha_n\le\bar f b_n(\bar\delta_{k-1})+\eta_n+\alpha_n$.

\emph{Asymptotics.} Under Assumption~\ref{ass:init} there are deterministic $\bar\delta_{0,n}\to0$ with
$\Prob(\delta_0>\bar\delta_{0,n})\to0$; take $\zeta_n\to0$ with $\log(1/\zeta_n)=o(n/\log^2n)$. Then
$\eta_n\log(8n)\to0$, $\alpha_n\to0$ and $b_n(\delta_n)\to0$ whenever $\delta_n\to0$ ($\log n/\kappa_n\to0$,
$c_n^+\to0$), so $F_n(\delta_n)\to0$ and $\bar\delta_{k,n}\to0$ for each fixed $k$; the hypotheses on $n$ hold
eventually, and the probability bound tends to one.
\end{proof}

\begin{proof}[Proof of Theorem~\ref{thm:recovery}]
Work on the event $E_n$ of the proof of Theorem~\ref{thm:joint}, where $\delta_{k-1}\le\bar\delta_{k-1}$ for every
$k\le K$, so $\{i:|v_i|>\tilde b_k\}\subseteq\mathcal D_n(\delta_{k-1})$ because $b_n$ is nondecreasing.
(a) Lemma~\ref{lem:wstep-uniform} at $\theta^{(k-1)}$ gives $w_i^{(k)}=1$ for clean and $w_i^{(k)}\le\alpha_n<\frac12$ for
flipped comparisons in $\mathcal D_n(\delta_{k-1})$, so such flips are flagged at either threshold
($w<1$ or $w<\frac12$) and such clean comparisons are not. Unflagged flips and flagged clean comparisons
therefore lie in $\{i:|v_i|\le \tilde b_k\}$, which on $H_n'$ and by Assumption~\ref{ass:linear}(iii) has at most
$(\bar f \tilde b_k+\eta_n)n$ elements; both $|\mathcal F^*|$ and $n-|\mathcal F^*|$ are at least $\rho_0n$.
(b) Under $W_\infty$ the same lemma gives $w_i^{(k)}=0$ on the flips of $\mathcal D_n(\delta_{k-1})$. Under $W_1$,
a flip in $\mathcal D_n(\delta_{k-2})$ has $w_i^{(k-1)}\le\alpha_n\le w_{\mathrm{sat}}$ (as $\lambda_n\le w_{\mathrm{sat}}^2$),
and if it also lies in $\mathcal D_n(\delta_{k-1})$ the lemma gives $w_i^{(k)}=0$; flips with
$|v_i|>\max(\tilde b_k,\tilde b_{k-1})$ lie in both sets. The count is as in (a).
(c) If $\min_i|v_i|>b_n(\bar\delta_0)\ge b_n(\delta_0)$ then $\mathcal D_n(\delta_0)=\{1,\dots,n\}$, and on
$G_n\cap\{\delta_0\le\bar\delta_0\}$ Lemma~\ref{lem:wstep-uniform} gives the statements about $w^{(1)}$. With
$w=w^*$ every term of $L_n$ is the corrected-label loss (proof of Lemma~\ref{lem:thetastep}), and
$\theta^{(1)}\ne0$ because $L_n(0;w^{(1)})=\log2$ exceeds $L_n(\theta_{\mathrm{ref}};w^{(1)})\le e^{-1}+\log\frac43$.
Under $W_\infty$, where $w^{(1)}=w^*$ has no fractional weights, the bound is $L_n(\theta_{\mathrm{ref}};w^{(1)})\le
b_n(\bar\delta_0)\le e^{-1}<\log2$, so $\alpha_n\le\frac14$ is needed only under $W_1$.
The same argument at $\theta^{(1)}$, with $w^{(1)}_i\le\alpha_n\le w_{\mathrm{sat}}$ on the flips, gives $w^{(2)}=w^*$
under $W_1$. Part (c) uses neither $H_n$, $H_n'$ nor Assumption~\ref{ass:linear}(iii)--(iv). The limits follow
from those in Theorem~\ref{thm:joint}, since $\tilde b_k\to0$ and $\eta_n\to0$.
\end{proof}

\begin{proof}[Proof of Corollary~\ref{cor:rates}]
Let $n$ be so large that $Cr_n^2\le1$, $\alpha_n\le\frac14$, $\lambda_n\le w_{\mathrm{sat}}^2$, and
$b_n(\bar\delta_k)\le\min\{e^{-1},b_0\}$, $\bar\delta_k<\pi$ for $k\le K$; these hold eventually by the bounds below.
For $c\le1$, $\lambda=e^c-1\le2c$, so $\alpha_n\le(2C)^{1/2}r_n$; and $c_n^+=2\,\mathrm{artanh}\,\alpha_n\le4\alpha_n$
for $\alpha_n\le\frac12$. Hence $b_n(\delta)\le(1+\epsilon)\log n/\kappa_n+B\delta+12B(2C)^{1/2}r_n$, and
$\eta_n\log(8n)=O(n^{-1/2}\log n)=O(r_n)$ for fixed $\zeta$. So $b_n(\bar\delta_0)=O(r_n)$ and, $b\mapsto b\log(4/b)$
being increasing, $\bar\delta_1=F_n(\bar\delta_0)=O(r_n\log(1/r_n))$; by induction
$b_n(\bar\delta_{k-1})=O(r_n\log^{k-1}(1/r_n))$ and $\bar\delta_k=O(r_n\log^k(1/r_n))$. The weight error and the rates
of Theorem~\ref{thm:recovery} are at most $\bar f b_n(\bar\delta_{k-1})+\eta_n+\alpha_n$ up to the factor $\rho_0^{-1}$.
The parameter $w_{\mathrm{sat}}$ enters the hypotheses only through $\lambda_n\le w_{\mathrm{sat}}^2$.
\end{proof}

\begin{proof}[Proof of Proposition~\ref{prop:fixedcut}]
Only Lemma~\ref{lem:wstep-uniform}(b) changes. Let $i\in\mathcal D_n(\delta)$ be a flip, the detectable set now
being defined with $\bar S c_n^\infty$ in place of $3Bc_n^+$. As in the proof of that lemma,
$\xi_i(u)\le-(|v_i|-B\delta)<-\bar Sc_n^\infty\le-\hat s(u)\,c_n^\infty$, so $z_i<-c_n^\infty$ and
$\tilde m_i=-t_i$ with $t_i>\tanh(c_n^\infty/2)=\tau_n$. By Lemma~\ref{lem:lla-limit}(b)--(c), $T_{t_i}(w)<w$ for
every $w\in(w_{\mathrm{sat}},1]$. The update is monotone and continuous in the current weight, so from any
starting weight in $[0,1]$ the iterates converge to a fixed point of $T_{t_i}$; none lies in
$(w_{\mathrm{sat}},1]$, and on $[0,w_{\mathrm{sat}}]$ the map is $0$, so the limit is $0$: $W_\infty(\theta,w)_i=0$.
Clean comparisons in $\mathcal D_n(\delta)$ keep weight $1$ as before. The weights are therefore exactly $w_i^*$ on
$\mathcal D_n(\delta)$, Lemma~\ref{lem:thetastep} applies with $\alpha=0$, and the proofs of
Theorems~\ref{thm:joint} and~\ref{thm:recovery} go through with $\alpha_n=0$ (the conditions
$\lambda_n\le w_{\mathrm{sat}}^2$ and $\alpha_n<\frac12$ enter them only through Lemma~\ref{lem:wstep-uniform}(b)
and the flip weight $\alpha$ of Lemma~\ref{lem:thetastep}). For the bound on $\tau_n$: $\varphi$ is
nondecreasing, so $\varphi(w)\le\varphi(1)=\lambda_n/(1+\lambda_n)$, and $w\ge w_{\mathrm{sat}}$, so
$M\le\lambda_n/\{(1+\lambda_n)w_{\mathrm{sat}}\}$; and $M/(2-M)\le M$ for $M\le1$. Finally
$2\,\mathrm{artanh}\,\tau\le4\tau$ for $\tau\le\frac12$, while $c_n^+=2\,\mathrm{artanh}\,\lambda_n^{1/2}\ge2\lambda_n^{1/2}$.
\end{proof}

\begin{proof}[Proof of Corollary~\ref{cor:oracle}]
\emph{A lower bound on the scale.} The sets $\{x:|x^\top u-q|\le s_0\}$, $u\in\Real^p$, $q\in\Real$, are
intersections of two half-spaces and form a class of finite VC dimension ($p$ fixed), so the event
$S_n=\{\sup_{u,q}|\frac1n\#\{i:|\phi_i^\top u-q|\le s_0\}-\Prob(|\phi^\top u-q|\le s_0)|\le\eta_S/2\}$ has
$\Prob(S_n)\to1$ \citep{vapnik1998}. Let $\theta\ne0$ with $\angle(\theta,\vartheta)\le\delta_S$, put
$u=\theta/\|\theta\|$ and $q=\mathrm{med}_j\,\xi_j(u)$. Since $|\xi_i(u)-q|=|\phi_i^\top u-Y_iq|$, a pair
with $|\xi_i(u)-q|\le s_0$ has $|\phi_i^\top u-q|\le s_0$ or $|\phi_i^\top u+q|\le s_0$, so on $S_n$, by
\eqref{eq:spread}, at most $2n(\frac14-\frac{\eta_S}2)=n(\frac12-\eta_S)<\frac n2$ pairs do. Whatever convention is used
for the median of an even number of values, both middle order statistics of $|\xi_i(u)-q|$ then exceed $s_0$,
so the median absolute deviation exceeds $s_0$, and $\hat s(\theta)>1.4826\,s_0\|\theta\|\ge s_0\|\theta\|$.

\emph{Good event.} Let $E_n$ be the intersection of $G_n$, $S_n$, the event in (i), and
$\{\varrho_n\le s_0c_n\}$ with $\varrho_n=(1+\epsilon)\log n/\kappa_n+B\delta_0$. By (i), (ii) and
$\Prob(G_n^c)\le n^{-\epsilon}$, $\Prob(E_n)\to1$. Take $n$ so large that $n\ge2p+2$,
$\lambda_n\le w_{\mathrm{sat}}^2$ and $B\delta_S>s_0c_n$; on $E_n$ the last gives $\delta_0<\delta_S$
(as $B\delta_0\le \varrho_n\le s_0c_n$). Work on $E_n$, at $\theta=\theta^{(0)}$, $u=\theta/\|\theta\|$, current
weights $w^{(0)}=\mathbf 1$; as in the proof of Lemma~\ref{lem:wstep-uniform}, $|\phi_i^\top u-v_i|\le B\delta_0$.

\emph{Clean comparisons stay at $w=1$.} Let $w_i^*=1$, so $Y_i=\widetilde Y_i$; $v_i\ne0$ almost surely. If
$\widetilde Y_i=\sign(v_i)$ then $\xi_i(u)=\sign(v_i)\phi_i^\top u\ge|v_i|-B\delta_0\ge-B\delta_0$. Otherwise,
on $G_n$, $\kappa_n|v_i|<(1+\epsilon)\log n$ and $\xi_i(u)\ge-|v_i|-B\delta_0>-\varrho_n$. In both cases
$\xi_i(u)\ge-\varrho_n$. If $\xi_i(u)\ge0$ then $z_i\ge0$; otherwise
$z_i=\xi_i(u)\|\theta\|/\hat s(\theta)>\xi_i(u)/s_0\ge-\varrho_n/s_0\ge-c_n$. So
$\tilde m_i=\tanh(z_i/2)\ge-\tanh(c_n/2)=-\lambda_n/(2+\lambda_n)$. With $w_i^{(0)}=1$ the adaptive weight is
$\varpi=\lambda_n$, and Lemma~\ref{lem:wstep} returns $w=1$; since the weight is unchanged, $1$ is a fixed
point of the update, and $W_1(\theta,\mathbf 1)_i=W_\infty(\theta,\mathbf 1)_i=1$.

\emph{Flips are reversed.} By (i) every flip lies in $\mathcal D_n(\delta_0)$, and
Lemma~\ref{lem:wstep-uniform}(b) gives $W_\infty(\theta,\mathbf 1)_i=0$ and $W_1(\theta,\mathbf 1)_i\le\alpha_n$.
This proves the weight statements in (a) and the first part of (b).

\emph{The refit is the oracle fit.} With $w=w^*$, the identity in the proof of Lemma~\ref{lem:thetastep}
(with $\alpha_i=0$) gives $\tfrac12+(w_i^*-\tfrac12)\tanh(Y_i\phi_i^\top\theta/2)=\sigma(\widetilde Y_i\phi_i^\top\theta)$
for every $i$, so $L_n(\cdot\,;w^*)=\ell_n$ and $\hat\theta(w^*)\in\Theta^{\mathrm{or}}_n$. This completes (a).

\emph{Second pass under $W_1$.} On the analogous event with $\delta_1$ in place of $\delta_0$ (probability
tending to one by the added hypotheses), the clean comparisons carry $w^{(1)}_i=1$ and the argument above
applies verbatim at $\theta^{(1)}$. A flip has $w^{(1)}_i\le\alpha_n\le w_{\mathrm{sat}}$ and lies in
$\mathcal D_n(\delta_1)$, so Lemma~\ref{lem:wstep-uniform}(b) gives $W_1(\theta^{(1)},w^{(1)})_i=0$. Hence
$w^{(2)}=w^*$ and $\theta^{(2)}\in\Theta^{\mathrm{or}}_n$.

\emph{Relative error.} For $\epsilon'>0$,
$\Prob(\|\theta^{(1)}-\theta^*_n\|>\epsilon'\|\theta^*_n\|)\le
\Prob(\sup_{\theta\in\Theta^{\mathrm{or}}_n}\|\theta-\theta^*_n\|>\epsilon'\|\theta^*_n\|)+\Prob(E_n^c)\to0$
under $W_\infty$, and likewise for $\theta^{(2)}$ under $W_1$.
\end{proof}

\section{Tuning and reporting}
\label{sec:tuning}

\paragraph{Parametrization.} We index the penalty by the cut $c$ and the saturation weight
$w_{\mathrm{sat}}=e^{-a\lambda}$, through $\lambda=e^c-1$ and $a=\log(1/w_{\mathrm{sat}})/\lambda$. The cut decides which
comparisons are flagged: on the first pass a comparison is down-weighted ($\hat w_i<1$) exactly when its $z$-scaled agreement is
below $-c$ robust standard deviations (Section~\ref{sec:zscale}). The saturation weight decides only what
happens to the most strongly contradicted comparisons. Reversal is reachable exactly when condition
\eqref{eq:Cinf} of Lemma~\ref{lem:lla-limit} holds, which at a given cut requires $w_{\mathrm{sat}}$ above a
threshold ($0.173$ at $c=0.85$ and $0.224$ at $c=1$); below it, flagged comparisons settle at an interior weight
(at $c=1$, $w_{\mathrm{sat}}=0.005$, no weight falls below $0.61$).

\paragraph{Reading the constants as thresholds.} Under $W_\infty$ the two constants are two thresholds on the
$z$-scaled agreement: a comparison is kept for $z_i\ge-c$, attenuated for $-c^\infty\le z_i<-c$, and reversed for
$z_i<-c^\infty$, where $\tanh(c^\infty/2)=\tau^{\mathrm{LLA}}$ (Lemma~\ref{lem:lla-limit},
Proposition~\ref{prop:fixedcut}). Below the reversal threshold of Table~\ref{tab:regimes}, $c^\infty=\infty$ and
the most contradicted comparisons stop at a floor weight. This suggests tuning $(c,c^\infty)$, two thresholds in
robust standard deviations, instead of $(c,w_{\mathrm{sat}})$: $w_{\mathrm{sat}}$ is recovered from $c^\infty$ by a
one-dimensional search, since $\tau^{\mathrm{LLA}}$ decreases in $w_{\mathrm{sat}}$ from $1$ to $\tanh(c/2)$. The
regime follows the goal. When flips are expected to be clear-cut and labels are to be corrected, choose a finite
$c^\infty$; Proposition~\ref{prop:fixedcut} then applies. For robust training on ambiguous preferences, choose
$c^\infty=\infty$; there $w_{\mathrm{sat}}$ matters little (at $c=1$ the floor weight is $0.61$ at
$w_{\mathrm{sat}}=0.005$ and $0.56$ at $0.1$), and it sets how much influence a flagged comparison keeps. At the
fitted model, the shares with $z_i<-c$ and $z_i<-c^\infty$ are the flagged and reversed shares; a reversed share well
above the plausible flip rate signals swamping.

\begin{table}[h]
\centering\small
\begin{tabular}{@{}llccccc@{}}
\toprule
& & \multicolumn{5}{c}{$w_{\mathrm{sat}}$}\\
\cmidrule(l){3-7}
$c$ & reversal needs $w_{\mathrm{sat}}>$ & $0.005$ & $0.1$ & $0.2$ & $0.3$ & $0.5$\\
\midrule
0.5 & 0.075 & \emph{f\,0.34} & $1.75$ & $0.79$ & $0.60$ & $0.50$ \\
0.85 & 0.173 & \emph{f\,0.54} & \emph{f\,0.48} & $2.48$ & $1.32$ & $0.91$ \\
1 & 0.224 & \emph{f\,0.61} & \emph{f\,0.56} & \emph{f\,0.49} & $1.87$ & $1.12$ \\
1.5 & 0.408 & \emph{f\,0.77} & \emph{f\,0.75} & \emph{f\,0.74} & \emph{f\,0.72} & $2.29$ \\
\bottomrule
\end{tabular}
\caption{Regimes of the $W_\infty$ weight step at a fixed cut. A number is the reversal threshold
$c^\infty$ (a comparison with $z_i<-c^\infty$ is reversed, one with $-c^\infty\le z_i<-c$ attenuated); an
entry f\,$x$ means reversal is impossible (condition \eqref{eq:Cinf} fails) and the most contradicted
comparisons stop at weight $x$. Computed by \texttt{src/regimes.py}.}
\label{tab:regimes}
\end{table}

\paragraph{Selection.} In the simulations every method's constants are chosen by clean-validation log-loss over
its grid (for RW-DPO, $c\in\{0.25,0.5,0.75,1,1.5,2\}$ and $w_{\mathrm{sat}}\in\{0.005,0.1,0.3\}$).
\emph{Fewer candidates.} When each fit is expensive, $w_{\mathrm{sat}}$ can be fixed at $0.005$ and the cut chosen
from a few values; on the language model we used $\{0.85,1,1.25\}$. In the simulations the
six grid cuts at $w_{\mathrm{sat}}=0.005$ come close to the full grid, whereas these three cuts alone are clearly
better than DPO but trail the full grid and several tuned baselines (Table~\ref{tab:sim-recipes}), so a wider set
is safer. The choice for the language model is
described in Section~\ref{sec:exp-rq2} and Appendix~\ref{app:llm}.

The fixed-$a$ convention with $a=1$ is the hard-thresholding end of the family (for an
orthogonal design, MCP with $a=1$ is hard thresholding, as in $\Theta$-IPOD
\citep{sheowen2011}): the attenuation band
$(\lambda/(2+\lambda),\tau^{\mathrm{LLA}}]$ is empty for $\lambda\le0.5$ and narrow
above it ($[0.333,0.389]$ at $\lambda=1$), so almost every comparison that leaves
$w=1$ goes to $w=0$ and graded attenuation largely disappears. The penalty is
nonconvex on the $w$-scale for $u\in(a\lambda-1,a\lambda)$ for every $(\lambda,a)$;
$\lambda<1/a$ only makes that region start at $w=1$, since on $u<a\lambda$
\[
F''_{\lambda,a}(w;m)=\frac{m^2}{[\tfrac12+(w-\tfrac12)m]^2}+\frac{\lambda-(1+u)/a}{w^2},
\qquad u=-\log w .
\]

\paragraph{Reporting.} Because $\hat w_i$ is an estimate of a per-comparison constant supported by a
single annotation, we recommend reporting the realized partition
$(\#\{\hat w_i=1\},\ \#\{0<\hat w_i<1\},\ \#\{\hat w_i=0\})$ alongside policy
metrics, and, under injected corruption, the rates of correct detection and of
\emph{swamping} --- reliable comparisons wrongly reversed. Swamping is the
quantity that decides whether a correction procedure is usable, and it is
routinely omitted.

\section{Simulation details}
\label{app:sim-details}

This section describes the simulations of Section~\ref{sec:exp-linear}: the model, three flip designs
(random, feature-dependent and clear-cut flips), the metrics, the methods with their implementation, and the tuning. Full results by flip rate and clearness are in Table~\ref{tab:sim-full}.

\paragraph{Model and data.} Each comparison has a feature vector $\phi_i\in\Real^p$ and margin
$d_i=\phi_i^\top\theta$; the true parameter is $\theta^*=\kappa\vartheta$ with $\|\vartheta\|=1$.
Clean labels are drawn from the Bradley--Terry model at $d_i^*=\phi_i^\top\theta^*$, and the labels
of the flip set $\mathcal F^*$ are reversed; $\rho=|\mathcal F^*|/n$ is the flip rate. A random
$20\%$ of the comparisons, with their \emph{clean} labels, form a validation split, and every
method is scored on a separate clean test set of $5000$ comparisons.

\paragraph{Metrics.} \emph{Angle error} is $\angle(\hat\theta,\theta^*)=\arccos\{\hat\theta^\top
\theta^*/(\|\hat\theta\|\|\theta^*\|)\}$, the error in the preference \emph{direction}, which is
what the theorem controls and what determines every ranking. \emph{Relative error} is
$\|\hat\theta-\theta^*\|/\|\theta^*\|$, which also penalizes a wrong scale. \emph{Test log-loss}
is the clean held-out negative log-likelihood (calibration). For methods that flag comparisons:
the \emph{detection rate} is the share of true flips that are flagged, the \emph{swamping rate}
the share of clean comparisons that are flagged, \emph{precision} the share of flagged comparisons
that are true flips. A comparison is \emph{flagged} by RW-DPO when $\hat w_i<1$, by trimmed logistic
regression when it is trimmed, and by the mean-shift fit when its shift is nonzero. PLC-DPO makes no explicit
flip decision: from its standardized margin it computes, for each pair, a posterior over clean, flipped and tied
states, and uses it only to mix the corresponding losses. To score its flags we recompute this posterior at the
final fit and count a pair as flagged when ``flipped'' is its most probable state.

\paragraph{Methods.} DPO (logistic regression on the reported labels); rDPO (the unbiased
loss of \citet{chowdhury2024rdpo}); cDPO (label smoothing); the noise-aware likelihood of
\citet{copas1988}, $-\log[(1-\varepsilon)\sigma(z)+\varepsilon\sigma(-z)]$, which is RW-DPO's model
with every weight equal to $1-\varepsilon$; trimmed logistic regression (least trimmed likelihood);
the case-specific mean-shift logistic regression with an $\ell_1$ penalty
\citep{tibsmanning2014,sheowen2011}; PLC-DPO \citep{cho2026plc} with its three published presets;
and RW-DPO (Algorithm~\ref{alg:rwdpo}, $z$-scaled agreement), with its cut
$c\in\{2,1.5,1,0.75,0.5,0.25\}$ and saturation weight $w_{\mathrm{sat}}\in\{0.005,0.1,0.3\}$ tuned jointly
($w_{\mathrm{sat}}=0.005$ only down-weights at these cuts, $0.1$ and $0.3$ also allow reversal); we also
report RW-DPO with $w_{\mathrm{sat}}=0.005$ fixed, which tunes the cut alone. Labels-known references are the oracle-label fit and the fit
that discards the flipped comparisons. \emph{Implementation:} every likelihood is minimized by L-BFGS with the
same small ridge term (below), started from the DPO fit; RW-DPO alternates one weight update with a full refit
of $\theta$ until no weight changes by more than $10^{-8}$, and PLC-DPO takes $400$ gradient steps.

\paragraph{Tuning.} Every method chooses its constants by the lowest log-loss on the same clean
validation split. For RW-DPO, the cut and saturation weight are chosen as described in
Appendix~\ref{sec:tuning}; the baselines choose their constant --- the flip rate $\varepsilon$ of
rDPO/cDPO/Copas, the trimmed share, the mean-shift penalty, PLC-DPO's preset --- from grids of
$3$--$7$ values. Tables report means with standard
errors; in each column the best method other than the oracle is in bold and methods within two
standard errors of it are underlined.

\paragraph{Ridge.} All studies use a ridge of $10^{-4}$ per comparison, small enough that the
oracle-label fit is essentially unbiased, so that relative error measures estimation error.
(With a ridge of the strength of a KL anchor, every fit is shrunk and methods that inflate the
scale can spuriously beat the oracle on relative error.)

\paragraph{Random flips.} Features are Gaussian; $n\in\{2000,8000\}$,
$p\in\{10,50\}$, the standard deviation of $d^*$ in $\{1.0,1.5,2.5\}$, and
$\rho\in\{0,0.1,0.2,0.3,0.4\}$, with $100$ replicates per setting. Flips are drawn independently
of the features, so the reported labels still depend on $\phi$ only through $\phi^\top\vartheta$;
for Gaussian features DPO's direction is then unbiased \citep{brillinger1982,liduan1989}, and the
robust methods can improve the scale and calibration but not the direction.
Averaged over $\rho>0$, the angle errors of all
methods lie within $0.374$--$0.391$ of DPO's $0.383$, while the relative error drops from $0.632$
for DPO to below $0.40$ for the best methods. RW-DPO has the lowest relative error ($0.381$) and
test log-loss ($0.545$), ahead of the Copas likelihood ($0.385$, $0.546$), which is the correctly
specified model when flips are uniform at random and has the lowest angle error ($0.374$ vs.\
$0.378$). RW-DPO's advantage grows with the flip rate (relative error $0.285$, $0.398$, $0.612$ at
$\rho=0.2,0.3,0.4$, against $0.289$, $0.404$, $0.634$ for Copas). Without
flips RW-DPO is close to the oracle ($0.178$ vs.\ $0.162$), whereas methods that assume a positive
flip rate pay for it (Copas $0.311$, PLC-DPO $0.250$, trimmed logistic $0.722$). Tuning $w_{\mathrm{sat}}$
as well as the cut helps: the joint RW-DPO beats RW-DPO with $w_{\mathrm{sat}}$ fixed in $85\%$ of
settings on test log-loss.

\paragraph{Feature-dependent flips.} The same grid, with flips concentrated along a
nuisance direction $u\perp\vartheta$, $\Pr(i\in\mathcal F^*)\propto\exp\{2\phi_i^\top u/\mathrm{sd}(\phi^\top u)\}$
--- an annotator bias tied to a feature unrelated to the true preference. With Gaussian features
the direction is again protected, and the methods behave as under random flips. With skewed features (centred exponential), the
reported labels pull DPO's direction away from $\vartheta$, and robust estimation pays off in the
direction itself (Figure~\ref{fig:sim-main}): RW-DPO has the lowest relative
error ($0.526$), angle error ($0.566$, against $0.591$ for DPO) and test log-loss ($0.579$) of all
methods. At $\rho=0.4$ no method recovers the direction.

\paragraph{Clear-cut flips.} Annotators can also err on easy comparisons (a careless
click or a swapped pair). Here flips are drawn only from comparisons with $|d_i^*|\ge t$, for
$t\in\{1.5,2.5,3.5,4.5\}$, with sd$(d^*)=4$, $\rho\in\{0.05,0.1\}$ (so that at most about $40\%$ of the
eligible comparisons are flipped and the direction stays identifiable), $n\in\{2000,8000\}$ and
$p\in\{10,50\}$. Clean comparisons keep their natural spread, close calls included. Results are in the last columns of
Table~\ref{tab:sim-main}. Tuned on clean validation data, RW-DPO's detection rate rises from $0.70$ to
$0.98$ as $t$ grows, while swamping falls from $0.012$ to $0.005$ and precision rises from $0.83$ to
$0.94$; its angle error ties the best method (trimmed logistic regression) at every $t$. PLC-DPO and
the mean-shift estimator detect nearly every flip but at precision $0.30$--$0.37$ and $0.18$--$0.22$,
flagging $14$--$18\%$ and $28$--$37\%$ of the clean comparisons. Validation tuning mostly down-weights
rather than reverses, because reversing even a few clean comparisons costs more in fit than reversing
a flip gains. Configured for flip recovery ($c=1$, $w_{\mathrm{sat}}=0.3$), the same update reverses
$27\%$, $37\%$, $56\%$ and $85\%$ of the flips at precision $0.98$--$0.99$, with the lowest angle error
of any setting at $t=4.5$ ($0.082$): detection consistency in the clearness direction, with the
direction estimate unharmed.

\paragraph{Summary.} With constants tuned on clean validation data, RW-DPO has the lowest estimation error and
log-loss under random flips and the lowest error on every metric under feature-dependent flips with
skewed features; its flags are among the most precise, and as flips become clearer its detection
approaches one. Its advantage is not universal: the correctly specified Copas likelihood has the
best angle under random flips, and flips on close calls cannot be undone by any method.

\section{Additional simulation results}
\label{app:sim}

\paragraph{Reversal on clear-cut flips.} Reversal pays when flips are clear-cut: at $c=1$, $w_{\mathrm{sat}}=0.3$ reverses a share of flips that grows with their clearness, from $27\%$ to $85\%$,
while setting at most $0.05\%$ of clean comparisons to $0$, with the same angle error as $w_{\mathrm{sat}}=0.005$.
When flips are expected to be clear-cut and the goal is to correct labels, $w_{\mathrm{sat}}=0.3$ is the better
choice; for estimation alone the two are equivalent.

\begin{table}[t]
\centering
\small
\caption{Full simulation results behind Table~\ref{tab:sim-main} (clean-validation tuning; 100 replicates per setting). (a) Relative error under random flips and relative and angle error under feature-dependent flips with skewed features, by flip rate $\rho$, means over the 12 settings (standard errors $\le0.008$); bold: best non-oracle; underlined: within 2 s.e. (b) Clear-cut flips, by clearness threshold $t$ (flips only where $|d_i^*|\ge t$, $\rho\in\{0.05,0.1\}$): detection / swamping / precision of the flags ($\hat w_i<1$ for RW-DPO) and angle error. Last row: RW-DPO configured for flip recovery, not tuned, counting exact reversals ($\hat w_i=0$).}
\label{tab:sim-full}
\textit{(a) Error by flip rate}\\[2pt]
\begin{adjustbox}{max width=\textwidth}\begin{tabular}{lccccccccccccc}
\toprule
 & \multicolumn{5}{c}{random flips: relative error} & \multicolumn{4}{c}{feature-dependent: relative error} & \multicolumn{4}{c}{feature-dependent: angle error} \\
Method & $\rho=0.0$ & $\rho=0.1$ & $\rho=0.2$ & $\rho=0.3$ & $\rho=0.4$ & $\rho=0.1$ & $\rho=0.2$ & $\rho=0.3$ & $\rho=0.4$ & $\rho=0.1$ & $\rho=0.2$ & $\rho=0.3$ & $\rho=0.4$ \\
\midrule
Oracle labels & 0.162 & 0.162 & 0.162 & 0.162 & 0.162 & 0.167 & 0.167 & 0.167 & 0.167 & 0.153 & 0.153 & 0.153 & 0.153 \\
\midrule
DPO (logistic) & \textbf{0.162} & 0.364 & 0.563 & 0.728 & 0.872 & 0.407 & 0.602 & 0.759 & 0.897 & 0.277 & 0.424 & 0.645 & \underline{1.016} \\
rDPO & 0.445 & \underline{0.219} & 0.303 & 0.436 & 0.667 & 0.296 & 0.456 & 0.631 & \textbf{0.841} & 0.278 & 0.427 & 0.647 & \textbf{1.016} \\
cDPO & 0.239 & 0.443 & 0.612 & 0.756 & 0.884 & 0.479 & 0.645 & 0.783 & 0.904 & 0.276 & 0.424 & 0.645 & \underline{1.016} \\
Copas (noise-aware) & 0.311 & \textbf{0.212} & \underline{0.289} & \underline{0.404} & 0.634 & \textbf{0.264} & 0.403 & \underline{0.606} & \underline{0.848} & \underline{0.250} & \underline{0.386} & \underline{0.627} & \underline{1.021} \\
Trimmed logistic & 0.722 & 0.304 & 0.304 & \underline{0.412} & \underline{0.620} & 0.348 & 0.406 & \underline{0.609} & 0.877 & \underline{0.247} & \underline{0.385} & \underline{0.635} & \underline{1.037} \\
Mean-shift $\ell_1$ & 0.466 & 0.252 & 0.302 & \underline{0.406} & \textbf{0.611} & 0.302 & 0.417 & \underline{0.607} & 0.853 & 0.259 & 0.403 & \underline{0.636} & \underline{1.023} \\
PLC-DPO & 0.250 & 0.222 & 0.300 & 0.424 & 0.630 & \underline{0.272} & 0.403 & \underline{0.609} & 0.861 & 0.255 & 0.394 & \underline{0.634} & \underline{1.027} \\
RW-DPO ($w_{\mathrm{sat}}$ fixed) & 0.178 & 0.232 & 0.298 & 0.417 & \underline{0.626} & \underline{0.270} & \underline{0.394} & \underline{0.608} & 0.853 & \underline{0.243} & \underline{0.379} & \underline{0.626} & \underline{1.023} \\
\textbf{RW-DPO} & 0.178 & 0.227 & \textbf{0.285} & \textbf{0.398} & \underline{0.612} & \underline{0.266} & \textbf{0.388} & \textbf{0.599} & 0.852 & \textbf{0.242} & \textbf{0.376} & \textbf{0.625} & \underline{1.022} \\
\bottomrule
\end{tabular}\end{adjustbox}

\vspace{6pt}
\textit{(b) Clear-cut flips by clearness threshold}\\[2pt]
\begin{adjustbox}{max width=\textwidth}\begin{tabular}{lcccccccc}
\toprule
Method & \multicolumn{2}{c}{$t=1.5$} & \multicolumn{2}{c}{$t=2.5$} & \multicolumn{2}{c}{$t=3.5$} & \multicolumn{2}{c}{$t=4.5$} \\
 & det./swamp./prec. & angle & det./swamp./prec. & angle & det./swamp./prec. & angle & det./swamp./prec. & angle \\
\midrule
Trimmed logistic & 0.73/0.013/0.82 & 0.112 & 0.86/0.008/0.89 & 0.108 & 0.94/0.004/0.95 & 0.095 & 0.98/0.002/0.97 & 0.087 \\
Mean-shift $\ell_1$ & 0.96/0.279/0.22 & 0.115 & 0.98/0.315/0.20 & 0.120 & 0.99/0.345/0.19 & 0.121 & 1.00/0.373/0.18 & 0.123 \\
PLC-DPO & 0.95/0.181/0.30 & 0.113 & 0.98/0.166/0.32 & 0.111 & 0.99/0.151/0.35 & 0.111 & 1.00/0.137/0.37 & 0.111 \\
\textbf{RW-DPO} & 0.70/0.012/0.83 & 0.111 & 0.87/0.011/0.87 & 0.106 & 0.95/0.008/0.91 & 0.092 & 0.98/0.005/0.94 & 0.087 \\
RW-DPO, flip recovery ($c{=}1$, $w_{\mathrm{sat}}{=}0.3$) & 0.27/0.000/0.98 & 0.112 & 0.37/0.000/0.99 & 0.113 & 0.56/0.000/0.99 & 0.089 & 0.85/0.001/0.99 & 0.082 \\
\bottomrule
\end{tabular}\end{adjustbox}
\end{table}
\begin{table}[t]
\centering
\small
\caption{RW-DPO's advantage in relative error is systematic: mean paired difference (RW-DPO minus method; negative favours RW-DPO) over the same replicates, with its standard error, every method tuned on clean validation data (Table~\ref{tab:sim-main}).}
\label{tab:sim-paired}
\begin{tabular}{lcc}
\toprule
Method & random flips & feature-dependent flips \\
\midrule
Copas (noise-aware) & $-0.0039$ {\scriptsize(0.0008)} & $-0.0044$ {\scriptsize(0.0007)} \\
Mean-shift $\ell_1$ & $-0.0120$ {\scriptsize(0.0009)} & $-0.0188$ {\scriptsize(0.0009)} \\
PLC-DPO & $-0.0135$ {\scriptsize(0.0008)} & $-0.0103$ {\scriptsize(0.0007)} \\
rDPO & $-0.0256$ {\scriptsize(0.0011)} & $-0.0299$ {\scriptsize(0.0011)} \\
Trimmed logistic & $-0.0293$ {\scriptsize(0.0010)} & $-0.0341$ {\scriptsize(0.0011)} \\
DPO (logistic) & $-0.2511$ {\scriptsize(0.0026)} & $-0.1405$ {\scriptsize(0.0019)} \\
cDPO & $-0.2930$ {\scriptsize(0.0026)} & $-0.1765$ {\scriptsize(0.0021)} \\
\bottomrule
\end{tabular}
\end{table}
\begin{table}[t]
\centering
\small
\caption{RW-DPO with a small candidate set of cuts and $w_{\mathrm{sat}}=0.005$: each of $c=1.25$, $1$, $0.85$ fixed in advance; the cut chosen from $\{0.85,1,1.25\}$ by clean-validation log-loss (the candidates used on the language model); the six grid cuts at $w_{\mathrm{sat}}=0.005$; and the full grid of Table~\ref{tab:sim-main} ($6$ cuts $\times$ $3$ values of $w_{\mathrm{sat}}$). Relative error $\|\hat\theta-\theta^*\|/\|\theta^*\|$; last column: share of replicates in which the three-cut choice beats DPO.}
\label{tab:sim-recipes}
\begin{adjustbox}{max width=\columnwidth}\begin{tabular}{llcccc|ccc|c}
\toprule
 & & & \multicolumn{3}{c|}{fixed cut} & \multicolumn{3}{c|}{chosen on validation} & beats \\
Flips & $\rho$ & DPO & $1.25$ & $1$ & $0.85$ & 3 cuts & 6 cuts & grid & DPO \\
\midrule
random & 0.1 & 0.364 & \textbf{0.241} & 0.277 & 0.409 & 0.219 & 0.232 & 0.227 & 89\% \\
 & 0.2 & 0.563 & 0.422 & \textbf{0.349} & 0.368 & 0.304 & 0.298 & 0.285 & 96\% \\
 & 0.3 & 0.728 & 0.622 & 0.537 & \textbf{0.497} & 0.465 & 0.417 & 0.398 & 97\% \\
 & 0.4 & 0.872 & 0.810 & 0.756 & \textbf{0.734} & 0.691 & 0.626 & 0.612 & 93\% \\
 & mean & 0.632 & 0.524 & \textbf{0.479} & 0.502 & 0.420 & 0.393 & 0.381 & 94\% \\
\midrule
feat.-dep. & 0.1 & 0.407 & \textbf{0.280} & 0.327 & 0.466 & 0.262 & 0.270 & 0.266 & 87\% \\
 & 0.2 & 0.602 & 0.477 & \textbf{0.436} & 0.481 & 0.397 & 0.394 & 0.388 & 95\% \\
 & 0.3 & 0.759 & 0.684 & \textbf{0.651} & 0.676 & 0.615 & 0.608 & 0.599 & 93\% \\
 & 0.4 & 0.897 & \textbf{0.878} & 0.892 & 0.954 & 0.861 & 0.853 & 0.852 & 74\% \\
 & mean & 0.666 & 0.580 & \textbf{0.577} & 0.644 & 0.534 & 0.531 & 0.526 & 87\% \\
\bottomrule
\end{tabular}\end{adjustbox}
\end{table}

\clearpage
\section{Additional language-model results}
\label{app:llm}

\paragraph{Training details.} We follow the PLC-DPO protocol \citep{cho2026plc}: LoRA rank $16$, $\beta=0.01$,
learning rate $10^{-5}$ with $10\%$ warm-up and cosine decay, effective batch $64$, and a fixed data order
shared by all methods. Training used single NVIDIA RTX 3090 and RTX 4090 GPUs (24 GB); a two-epoch RW-DPO run takes
about $5.4$ hours on an RTX 3090 and $3.2$ on an RTX 4090, and the simulations ran on CPUs. The data
(UltraFeedback), the base and reference model (a Qwen2.5-1.5B model fine-tuned on UltraChat), the judge
(Skywork-Reward-V2-Llama-3.1-8B) and the PLC-DPO code are public and were used under the licences stated on their
release pages; no new data were collected.

\paragraph{Implementation of RW-DPO on the language model.} Algorithm~\ref{alg:rwdpo} is adapted to stochastic
training. Instead of a full DPO fit, the first $7\%$ of the updates ($43$ of $621$ in two epochs) are plain DPO.
The weight step runs on each mini-batch when it is visited, once per epoch, and is iterated at the current
margins until the weights change by less than $10^{-3}$. The robust scale $\hat s$ is the MAD of a running window
of the $2000$ most recent label-oriented margins, computed before the current batch is added. The consistency
theory of Section~\ref{sec:feasible} covers the full-data scheme of the linear-index model, not this stochastic
variant. Two-epoch results use the adapter saved at the end of the second epoch.

\paragraph{Development cell.} On $\rho=0.2$, seed $0$, the cuts $0.85$ and $1$ gave validation log-loss $0.615$
and $0.619$ and win rates against one-epoch DPO $72.7$ and $72.8$; a smaller cut, $c=0.5$, gave $0.773$ and $69.7$.
The third candidate, $c=1.25$, was run on this cell after the grid, to complete the comparison: validation
log-loss $0.617$ and win rate $72.8$, so the choice of $c=0.85$ stands (test log-loss, not used for the choice, is
$0.644$ at $c=1.25$ against $0.660$ at $c=0.85$).

\paragraph{Further results.} Table~\ref{tab:llm-clean} gives every method at one and two epochs against
clean-label DPO. Table~\ref{tab:llm-prefacc} gives preference accuracy on the clean test set; it does not track generation quality
here, since cDPO and DPO have the highest accuracy and the lowest win rates.

\begin{table}[t]
\centering
\small
\caption{Win rate (\%) against DPO trained on clean labels for two epochs, averaged over its three seeds (the reference of Table~\ref{tab:llm-summary}); $50$ = as good as clean-label DPO. UltraFeedback with a share $\rho$ of flipped labels, Qwen2.5-1.5B (LoRA), mean $\pm$ s.d.\ over three seeds (fewer where marked). }
\label{tab:llm-clean}
\begin{tabular}{llccccc}
\toprule
Epochs & Method & $\rho=0.0$ & $\rho=0.1$ & $\rho=0.2$ & $\rho=0.3$ & mean over $\rho>0$ \\
\midrule
1 & DPO & 37.8\,{\scriptsize$\pm$0.2} & 33.6\,{\scriptsize$\pm$0.5} & 30.1\,{\scriptsize$\pm$0.7} & 26.4\,{\scriptsize$\pm$0.2} & 30.0 \\
 & rDPO ($\varepsilon=\rho$) & -- & 34.1\,{\scriptsize$\pm$0.5} & 30.1\,{\scriptsize$\pm$0.1} & 26.8\,{\scriptsize$\pm$0.3} & 30.4 \\
 & cDPO ($\varepsilon=\rho$) & -- & 32.7\,{\scriptsize$\pm$0.7} & 29.5\,{\scriptsize$\pm$0.7} & 25.7\,{\scriptsize$\pm$0.2} & 29.3 \\
 & PLC-DPO & 48.5\,{\scriptsize$\pm$0.7} & 44.8\,{\scriptsize$\pm$0.9} & 40.2\,{\scriptsize$\pm$1.5} & 32.6\,{\scriptsize$\pm$0.1} & 39.2 \\
 & \textbf{RW-DPO} & 48.8\,{\scriptsize$\pm$0.3} & 46.2\,{\scriptsize$\pm$0.4} & 40.1\,{\scriptsize$\pm$0.5} & 31.6\,{\scriptsize$\pm$2.2} & 39.3 \\
\midrule
2 & DPO & 50.0\,{\scriptsize$\pm$1.1} & 45.9\,{\scriptsize$\pm$1.0} & 41.7\,{\scriptsize$\pm$1.8} & 34.5\,{\scriptsize$\pm$0.2} & 40.7 \\
 & rDPO ($\varepsilon=\rho$) & -- & 49.4\,{\scriptsize$\pm$0.6} & 46.6\,{\scriptsize$\pm$0.8} & 40.0\,{\scriptsize$\pm$1.3} & 45.3 \\
 & cDPO ($\varepsilon=\rho$) & -- & 43.9\,{\scriptsize$\pm$0.9} & 36.6\,{\scriptsize$\pm$0.8} & 29.9\,{\scriptsize$\pm$0.6} & 36.8 \\
 & PLC-DPO & 53.8\,{\scriptsize$\pm$0.5} & 53.7\,{\scriptsize$\pm$0.6} & 54.2\,{\scriptsize$\pm$1.7} & 54.0\,{\scriptsize$\pm$1.8} & 54.0 \\
 & \textbf{RW-DPO} & 52.8\,{\scriptsize$\pm$0.2} & 53.4\,{\scriptsize$\pm$1.2} & 53.7\,{\scriptsize$\pm$1.1} & 55.1\,{\scriptsize$\pm$2.9} & 54.1 \\
\bottomrule
\end{tabular}
\end{table}
\begin{table}[t]
\centering
\small
\caption{Preference accuracy (\%) on the clean held-out test set. UltraFeedback with a share $\rho$ of flipped labels, Qwen2.5-1.5B (LoRA), mean $\pm$ s.d.\ over three seeds (fewer where marked). }
\label{tab:llm-prefacc}
\begin{tabular}{llccccc}
\toprule
Epochs & Method & $\rho=0.0$ & $\rho=0.1$ & $\rho=0.2$ & $\rho=0.3$ & mean over $\rho>0$ \\
\midrule
1 & DPO & 64.7\,{\scriptsize$\pm$0.2} & 65.5\,{\scriptsize$\pm$0.1} & 65.7\,{\scriptsize$\pm$0.4} & 65.9\,{\scriptsize$\pm$0.2} & 65.7 \\
 & rDPO ($\varepsilon=\rho$) & -- & 65.3\,{\scriptsize$\pm$0.2} & 65.7\,{\scriptsize$\pm$0.0} & 66.0\,{\scriptsize$\pm$0.9} & 65.7 \\
 & cDPO ($\varepsilon=\rho$) & -- & 65.5\,{\scriptsize$\pm$0.1} & 65.7\,{\scriptsize$\pm$0.3} & 65.7\,{\scriptsize$\pm$0.9} & 65.6 \\
 & PLC-DPO & 61.6\,{\scriptsize$\pm$0.1} & 62.0\,{\scriptsize$\pm$0.3} & 63.1\,{\scriptsize$\pm$0.1} & 64.1\,{\scriptsize$\pm$0.3} & 63.1 \\
 & \textbf{RW-DPO} & 62.4\,{\scriptsize$\pm$0.1} & 61.9\,{\scriptsize$\pm$0.3} & 62.3\,{\scriptsize$\pm$1.1} & 63.0\,{\scriptsize$\pm$1.5} & 62.4 \\
\midrule
2 & DPO & 66.7\,{\scriptsize$\pm$0.2} & 65.5\,{\scriptsize$\pm$0.3} & 64.8\,{\scriptsize$\pm$0.4} & 65.3\,{\scriptsize$\pm$0.4} & 65.2 \\
 & rDPO ($\varepsilon=\rho$) & -- & 65.3\,{\scriptsize$\pm$0.5} & 64.2\,{\scriptsize$\pm$0.2} & 64.3\,{\scriptsize$\pm$0.6} & 64.6 \\
 & cDPO ($\varepsilon=\rho$) & -- & 65.6\,{\scriptsize$\pm$0.4} & 66.0\,{\scriptsize$\pm$0.4} & 66.9\,{\scriptsize$\pm$0.5} & 66.2 \\
 & PLC-DPO & 64.0\,{\scriptsize$\pm$0.1} & 63.6\,{\scriptsize$\pm$0.2} & 62.7\,{\scriptsize$\pm$0.4} & 61.0\,{\scriptsize$\pm$0.3} & 62.4 \\
 & \textbf{RW-DPO} & 65.1\,{\scriptsize$\pm$0.1} & 64.3\,{\scriptsize$\pm$0.3} & 63.4\,{\scriptsize$\pm$0.6} & 61.2\,{\scriptsize$\pm$0.3} & 63.0 \\
\bottomrule
\end{tabular}
\end{table}
\begin{table}[t]
\centering
\small
\caption{RW-DPO's flags on the language model (two epochs, $c=0.85$, $w_{\mathrm{sat}}=0.005$), by flip rate: share of comparisons with $\hat w_i<1$ in the weights used in the second epoch, detection (share of flips flagged), swamping (share of clean comparisons flagged) and precision (share of flags that are flips); means over seeds. Flip masks are used only for this evaluation.}
\label{tab:llm-flags}
\begin{tabular}{ccccc}
\toprule
$\rho$ & flagged & detection & swamping & precision \\
\midrule
0.1 & 0.128 & 0.308 & 0.108 & 0.24 \\
0.2 & 0.152 & 0.307 & 0.114 & 0.40 \\
0.3 & 0.179 & 0.296 & 0.129 & 0.50 \\
\bottomrule
\end{tabular}
\end{table}
\begin{table}[t]
\centering
\small
\caption{Where each method's correction lands on the language model (two epochs), by flip rate. The correction of a comparison is the weight RW-DPO removes, $1-\hat w_i$ (weights used in the second epoch), and for PLC-DPO the weight its loss puts on the reversed label, $m_i\,p_{\mathrm{flip},i}$ with mixture weight $m_i=\gamma\cdot\text{confidence}_i^{\,\kappa}$, recomputed with its own routing formula from the final model (its routing is not stored during training). Applied: mean correction; on flips: share of the total correction that falls on flipped comparisons (equal to $\rho$ if spread at random); ratio: mean correction on a flipped comparison over that on a clean one (1 if spread evenly). Both corrections are the weight placed on the reversed label, in the likelihood for RW-DPO and in the loss for PLC-DPO, so their levels are comparable only roughly; the share on flips and the ratio are scale-free. Mean $\pm$ s.d.\ over three seeds.}
\label{tab:llm-landing}
\begin{adjustbox}{max width=\columnwidth}\begin{tabular}{c|ccc|ccc}
\toprule
 & \multicolumn{3}{c|}{RW-DPO} & \multicolumn{3}{c}{PLC-DPO} \\
$\rho$ & applied & on flips & ratio & applied & on flips & ratio \\
\midrule
0.1 & 0.03\,{\scriptsize$\pm$0.00} & 0.26\,{\scriptsize$\pm$0.01} & 3.21\,{\scriptsize$\pm$0.15} & 0.11\,{\scriptsize$\pm$0.00} & 0.10\,{\scriptsize$\pm$0.00} & 0.99\,{\scriptsize$\pm$0.01} \\
0.2 & 0.04\,{\scriptsize$\pm$0.00} & 0.43\,{\scriptsize$\pm$0.02} & 2.99\,{\scriptsize$\pm$0.24} & 0.12\,{\scriptsize$\pm$0.00} & 0.20\,{\scriptsize$\pm$0.00} & 1.00\,{\scriptsize$\pm$0.01} \\
0.3 & 0.05\,{\scriptsize$\pm$0.00} & 0.51\,{\scriptsize$\pm$0.02} & 2.41\,{\scriptsize$\pm$0.17} & 0.10\,{\scriptsize$\pm$0.00} & 0.30\,{\scriptsize$\pm$0.00} & 1.02\,{\scriptsize$\pm$0.00} \\
\bottomrule
\end{tabular}\end{adjustbox}
\end{table}

\end{document}